\documentclass[english]{IEEEtran}
\usepackage[T1]{fontenc}
\usepackage[latin9]{inputenc}
\usepackage{color}
\usepackage{babel}
\usepackage{float}
\usepackage{amsmath}
\usepackage{amsthm}
\usepackage{amssymb}
\usepackage{stmaryrd}
\usepackage{stackrel}
\usepackage{graphicx}
\usepackage{geometry}
\usepackage{multirow}
\PassOptionsToPackage{normalem}{ulem}
\usepackage{ulem}
\usepackage[]
 {hyperref}

\makeatletter

\newcommand{\lyxmathsym}[1]{\ifmmode\begingroup\def\b@ld{bold}
  \text{\ifx\math@version\b@ld\bfseries\fi#1}\endgroup\else#1\fi}

\floatstyle{ruled}
\newfloat{algorithm}{tbp}{loa}
\providecommand{\algorithmname}{Algorithm}
\floatname{algorithm}{\protect\algorithmname}

\providecommand{\tabularnewline}{\\}
\floatstyle{ruled}
\newfloat{algorithm}{tbp}{loa}
\providecommand{\algorithmname}{Algorithm}
\floatname{algorithm}{\protect\algorithmname}

\theoremstyle{definition}
\newtheorem{example}{\protect\examplename}
\theoremstyle{definition}
\newtheorem{problem}{\protect\problemname}
\theoremstyle{definition}
\newtheorem{defn}{\protect\definitionname}
\theoremstyle{plain}
\newtheorem{thm}{\protect\theoremname}
\theoremstyle{plain}

\theoremstyle{plain}

\newtheorem{rem}{Remark}  % 定义rem环境，显示为 "Remark 1" 样式

\usepackage{cite}
\usepackage{algpseudocode}
\usepackage{setspace}
\IEEEoverridecommandlockouts

\makeatother

\usepackage{babel}
\providecommand{\corollaryname}{Corollary}
\providecommand{\definitionname}{Definition}
\providecommand{\examplename}{Example}
\providecommand{\lemmaname}{Lemma}
\providecommand{\problemname}{Problem}
\providecommand{\theoremname}{Theorem}

\begin{document}
\title{Temporal Logic Guided Universal Task Representations for Reinforcement Learning
\thanks{{H. Zhang, Z. Zhou, and Z. Kan (Corresponding Author) are with the Department
of Automation at the University of Science and Technology of China,
Hefei, Anhui, China, 230026. 

H. Zhang and Z. Zhou are also with the Anhui Provincial Key Laboratory of Humanoid Robots, Hefei, China, 230000.

This work was supported in part by the National Natural Science Foundation
of China under Grant U25A20473, and in part by the Open Project of Anhui Provincial Key Laboratory of Humanoid Robots under Grant JH-RX-2025-0201. The AI-driven experiments, simulations
and model training were performed on the robotic AI-Scientist platform of
Chinese Academy of Sciences.}
}}
\author{Hao Zhang, Zhangli Zhou, and Zhen Kan}
\maketitle
\begin{abstract}
Task guided agents demonstrate strong performance in a wide range of complex
tasks. However, most existing task representation algorithms are tailored to specific contexts and struggle to generalize across diverse scenarios. Moreover, they typically depend on gradient signals from reinforcement learning controllers to update their weights, which can degrade both representation quality and learning efficiency.  \textcolor{black}{To overcome these limitations,
we propose LOTUS, a temporal }\textcolor{black}{\uline{lo}}\textcolor{black}{gic
inspired }\textcolor{black}{\uline{u}}\textcolor{black}{niversal }\textcolor{black}{\uline{t}}\textcolor{black}{ask
repre}\textcolor{black}{\uline{s}}\textcolor{black}{entation framework
that can be seamlessly integrated into any RL algorithm to enhance
agent performance across diverse task settings.} Specifically, we design a novel task representation architecture capable of modeling relationships and extracting task semantics from LTL formulas. We further introduce a more effective update mechanism that treats the LTL encoder as a policy, thereby improving representation capacity. To enhance stability and robustness, LOTUS leverages the bisimulation metric, which provides theoretical guarantees for LTL representation, including behavioral equivalence, optimality fidelity, and trajectory robustness. Experimental results show that LOTUS outperforms most existing methods in learning efficiency, generalization capability, and representation quality. \textcolor{black}{Specifically,
LOTUS accelerates convergence over 20\% in single-task scenarios,
achieves a 15\%--45\% higher success rate in unseen manipulation
tasks, and improves generalization performance over 25\% in complex
multi-task environments with increased sub-goal depth or conjunctions.} The corresponding code, videos, and appendix are available at: \href{https://lotus-website.github.io/}{https://lotus-website.github.io/}. 

\global\long\def\prog{\operatorname{prog}}%
\global\long\def\argmax{\operatorname{argmax}}%
\global\long\def\argmin{\operatorname{argmin}}%
\end{abstract}

\section{Introduction}

One of the ultimate goals in robot learning is to enable robots to comprehend and plan effectively across diverse task scenarios based on task instructions. While deep reinforcement learning (DRL)
has shown great potentials \cite{Sutton2018}, it often struggles
to handle complex tasks due to challenges in exploration and semantic understanding. Recent research has addressed these issues by parsing task instructions or leveraging task semantics to guide robots in diverse settings. For example, task encoders have been employed to extract task semantics, thereby helping agents explore and learn the necessary skills at different task stages \cite{vaezipoor2021ltl2action,zhang2023exploiting,yalcinkaya2024compositional}.
\textcolor{black}{However, much of the existing work focuses on single-task learning, whereas the ability to generalize across multiple tasks is equally critical for assessing an agent\textquoteright s planning capabilities \cite{farhadi2024domain, botteghi2025unsupervised}.} Recent advances have shown promise in enabling zero-shot generalization in multi-task environments through task relationship modeling  \cite{zhang2021survey,cheng2023multi,zhu2024multi}
or temporal abstraction frameworks \cite{yoo2022learning,qiu2023instructing,wang2023task}.
Nevertheless, these approaches often fall short in both single-task
and multi-task scenarios, facing challenges related to training efficiency,
generalization performance, and interpretability.

An effective approach to improving policy performance across diverse task scenarios is to guide agents using task semantics. Research in this area can be broadly divided into two directions: (1) task planning driven by foundational models and (2) semantic-guided learning through task representation modules. The first line of work leverages large language models (LLMs) to decompose complex tasks into structured subtasks and generate coherent instructions, thereby enabling robots to complete tasks more effectively \cite{cao2024survey}. This approach has already shown promising results in motion planning \cite{ding2023task,chen2024autotamp,zhou2024avatargpt} and object manipulation \cite{wu2025momanipvla,wen2025tinyvla,liu2024robomamba}. However, challenges remain, including semantic inconsistency, reliance on prompt engineering, and limited planning capability in handling complex task instructions. The second line of work employs neural network based task encoders to represent task semantics from instructions, thereby enhancing agent training efficiency. For instance, graph neural networks (GNNs) \cite{scarselli2008graph,schlichtkrull2018modeling,brody2021attentive} excel at modeling relationships among objects or task targets, providing agents with effective feedback and guidance during execution. Similarly, Transformer-based architectures \cite{vaswani2017attention} leverage self-attention to capture global dependencies in complex tasks and integrate multimodal information \cite{wang2024navformer,chen2023transformer,firoozi2025foundation}. Despite these advances, existing methods rarely address task scenarios that involve intricate logical dependencies and temporal constraints. 

Recent research has increasingly focused on policy
learning for complex tasks through formal methods. As a formal language,
linear temporal logic (LTL) has been widely adopted to describe complex
logic and temporal constraints due to its rich expressivity \cite{Baier2008}. Approaches that guide agents with LTL can be broadly categorized into two groups: (1) automaton-based methods and (2) task-encoder-based methods. Automaton-based methods construct task-specific automata to represent logical and temporal structures, enabling agents to perform effectively across various task types \cite{Li2019,Cai2021b,Cai2021c}. However, these approaches typically encode automaton states as ordered indices, a representational form that provides limited benefit to RL performance. To enhance the utility of task representations, task encoder based methods leverage LTL progression \cite{tuli2022learning} to track task stages and encode the corresponding LTL formulas into semantic representations. These representations then guide agents toward more effective exploration and learning \cite{vaezipoor2021ltl2action,zhang2023exploiting,yalcinkaya2024compositional}. Nonetheless, existing LTL encoders are either specialized for generalization in multi-task settings or limited to single-task learning. \textcolor{black}{Consequently, no universal framework currently
exists that can achieve robust performance across diverse task scenarios
\cite{rezaei2024visual,noori2025latent}}.

Meanwhile, recent advances have been devoted to improving
the generalization performance through goal-conditional RL (GCRL)
with LTL instructions \cite{qiu2023instructing,deepltl}. Such
methods typically follow a two-phase process: first training sub-task policies, and then invoking them at the task-planning level to accomplish tasks efficiently. However, methods that rely on LTL representations generally update the encoder indirectly through RL, making the task representations unstable and potentially degrading the agent\textquoteright s policy performance.

The main motivation of this work is to design an LTL encoder architecture that can not only guide agents efficiently in accomplishing complex tasks across diverse scenarios, but also update its weights directly through a well-designed loss function to improve training efficiency. To this end, we
propose LOTUS, a temporal \uline{lo}gic
inspired \uline{u}niversal \uline{t}ask
repre\uline{s}entation framework
that introduces a new LTL encoder architecture along with a tailored
loss function to enhance learning efficiency in guiding the agent. Specifically, LOTUS incorporates a novel LTL encoder, termed Relational Graph Transformer Networks (RGTN), which effectively models the relationships among atomic propositions in LTL formulas, extracts task semantics, and accelerates policy learning in both single-task and multi-task settings. To eliminate dependence on RL modules for updating the encoder, LOTUS further introduces TLPG, an update mechanism that leverages \uline{p}olicy \uline{g}radient to accelerate
the convergence of the L\uline{TL} encoder by constructing the corresponding task-level MDP. Stability of the TLPG update is further reinforced through the bisimulation metric. In addition, LOTUS provides explicit theoretical guarantees for LTL representations, including behavioral equivalence, optimality fidelity, and trajectory robustness.

The main contributions are summarized as follows:
\begin{itemize}
\item We propose a task representation algorithm that can be seamlessly integrated into any RL framework, boosting sampling efficiency and enhancing performance and generalization across both single and multi-task learning.
\item We design the RGTN as a task representation module that efficiently encodes LTL instructions, extracts task semantics through specialized attention mechanisms.
\item To our knowledge, we are the first to directly update the LTL encoder as a policy using rewards, with theoretical guarantees based on the bisimulation metric.
\item Extensive experiments demonstrate that LOTUS significantly outperforms baseline methods in learning efficiency, generalization, and representation quality.
\end{itemize}

\textcolor{black}{The remainder of this work is organized as follows.
Sec. \ref{sec:Pre} presents preliminaries and problem formulation.
Sec. \ref{sec:Algorithm Design} elaborates the LOTUS framework design,
encompassing Relational Graph Transformer Network (RGTN) encoder,
direct update mechanism (TLPG), bisimulation metric optimization,
and theoretical guarantees. Sec. \ref{sec:EXPERIMENTS} validates
LOTUS via extensive experiments, including multiple task scenarios,
generalization evaluation, ablation studies, and real-world performance.
Sec. \ref{sec:Conclusion} concludes with key contributions and future
directions.}

\section{Preliminaries\label{sec:Pre}}

\textcolor{black}{To enhance the readability, we summarize the main
symbol appointment of this work in our \href{https://lotus-website.github.io}{website}.}

\subsection{sc-LTL and LTL Progression\label{subsec:sc-LTL}}

Co-safe LTL (sc-LTL) is a subclass of LTL that can be satisfied by
finite-horizon state trajectories \cite{Kupferman2001}. Since sc-LTL
is suitable to describe robotic instructions (e.g., approach the bucket,
grasp the handle, and then lift it up to the table), this work focuses
on sc-LTL. An sc-LTL formula is built on a set of atomic propositions
$\Pi$ that can be true or false, standard Boolean operators such
as $\wedge$ (conjunction), $\lor$ (disjunction), and $\lnot$ (negation),
temporal operators such as $\bigcirc$ (next), $\diamondsuit$ (eventually),
and $\cup$ (until). The semantics of an sc-LTL formula are interpreted
over a word $\boldsymbol{\sigma}=\sigma_{0}\sigma_{1}...\sigma_{n}$,
which is a finite sequence with $\sigma_{i}\in2^{\Pi}$, $i=0,\ldots,n$,
where $2^{\Pi}$ represents the power set of $\Pi$. Denote by $\left\langle \boldsymbol{\sigma},i\right\rangle \vDash\varphi$
if the sc-LTL formula $\varphi$ holds from position $i$ of $\boldsymbol{\sigma}$.
More detailed explanations and examples can be found in \cite{Baier2008}. 

LTL formulas can also be progressed along a sequence of truth assignment
\cite{tuli2022learning}. Specifically, give an LTL formula $\varphi$
and a word $\sigma=\sigma_{0}\sigma_{1}...$, the LTL progression
$\mathrm{prog}(\sigma_{i},\varphi)$ at step $i$ is $\mathrm{prog}(\sigma_{i},p)=\mathrm{True}$
if $p\in\sigma_{i}$, where $p\in\Pi$ and $\mathrm{prog}(\sigma_{i},p)=\mathrm{False}$
otherwise. The operator $\mathrm{prog}$ takes an LTL formula $\varphi$
and the current label $\sigma_{i}$ as input at each step, and outputs
a formula indicating the remaining instructions to be addressed.

\subsection{Labeled MDP and Reinforcement Learning}

When performing the sc-LTL task $\varphi$ in RL, the interaction
between the robot and the environment can be modeled by a labeled
MDP $\mathcal{M}_{e}=\left(S,T,A,p_{e},\Pi,L,R,\gamma,\mu\right)$,
where $S$ is the state space, $T\subseteq S$ is a set of terminal
states, $A$ is the action space, $p_{e}(s'|s,a)$ is the transition
probability from $s\in S$ to $s'\in S$ under action $a\in A$, $\Pi$
is a set of atomic propositions indicating the properties associated
with the states, $L:S\rightarrow2^{\Pi}$ is the labeling function,
$R:S\times A\rightarrow\mathbb{R}$ is the reward function, $\gamma\in\left(0,1\right]$
is the discount factor, $\mu$ is the initial state distribution.
The labeling function $L$ can be seen as a set of event detectors
that trigger when $p\in\Pi$ presents in the environment, allowing
the robot to determine whether or not an LTL specification is satisfied. 

For any task $\varphi$, the robot interacts with the environment
following a deterministic policy $\pi_{e}$ over $\mathcal{M}_{e}$,
which $\pi_{e}:S\rightarrow A$ maps each state $s_{t}$ to an action
$a_{t}$ over the action space $A$. The robot then receives a reward
by $r_{t}=R(s_{t},a_{t})$. Given a task $\varphi$, the goal of the
agent is to learn an optimal policy $\pi_{e}^{*}(a|s)$ that maximizes
the expected discounted return $\underset{\tau\sim\pi_{e}}{\mathbb{E}}\left[\stackrel[k=0]{\infty}{\sum}\gamma^{k}r_{t+k}\mid S_{t}=s,A_{t}=a\right]$
starting from any state $s\in S,$ action $a\in A$ and time step $t$. 

\subsection{Bisimulation Metrics in RL\label{subsec:BS_in_RL}}

In the high-dimensional state space of RL, clustering similar states into the same set can improve robustness
\cite{zhang2020learning}. Bisimulation is a state abstraction technique that groups two states $s_{i}$ and $s_{j}$ as equivalent
if they exhibit identical behavior. To quantify the similarity between two states, the bisimulation metric is defined using the $p$-th Wasserstein distance $W_{p}\left(P_{1},P_{2}\right)$ \cite{villani2008optimal} between two probability distributions $P_{1}$ and $P_{2}$.
\begin{defn}
\label{def:bs_in_rl}(Bisimulation metric \cite{ferns2011bisimulation})
Given a reward function $R:S\times A\rightarrow[0,1]$
and a weighting coefficient $c\in(0,1)$ for continuous MDPs, the following metric 
\[
\begin{array}{cc}
d\left(s_{i},s_{j}\right)= & \underset{a\in\mathcal{A}}{\max}\left(1-c\right)\left|R\left(s_{i},a\right)-R\left(s_{j},a\right)\right|\\
 & +cW_{1}\left(p_{e}\left(\cdot\mid s_{i},a\right),p_{e}\left(\cdot\mid s_{j},a\right)\right)
\end{array}
\]
exists and is unique, where $p_{e}\left(\cdot\mid\cdot,a\right)$ is the
transition probability from current state to the next state under
action $a\in A$. 
% and $d\left(s_{i},s_{j}\right)$ represents the bisimulation metric between the state $s_{i}$ and $s_{j}$.
\end{defn}

\subsection{Challenges and Problem Formulation}

To elaborate the problem and challenges, the following running example
will be used throughout the work

\begin{example}
\label{example1}

Consider diverse task scenarios \cite{vaezipoor2021ltl2action,nasiriany2022augmenting,bao2023dexart}
shown in Fig. \ref{fig:Outline4Example}, \textcolor{black}{in which
the robot need to complete different complex tasks shown in Table
\ref{tab:task_in_ex} by the guidance of the same task representation module when given
different sets of propositions.} Different colored
trajectories present possible solutions for the robot in completing
the tasks. \textcolor{black}{Detailed task descriptions are available
on our \href{https://lotus-website.github.io}{website}}\textcolor{black}{.}
\end{example}

There are three main challenges here in Example \ref{example1}. The
first challenge is how to design the LTL encoder so that its output
LTL representation can effectively guide the agent to learn in both
single-task and multi-task environments. The second challenge is how
to design a more effective update way for the task representation
module to avoid the previous limitation of indirectly updating the
weights through the RL controller to further improve the learning
efficiency of the agent. The third challenge is how to further improve
the representation quality of stability and robustness during the
training process.

To this end, our goal is to design a universal task representation
module $\varTheta$ with weights $\xi$, which not only facilitates
the agent's leaning in diverse task scenarios, but also updates its
representation more effectively and stably. By exploiting the LTL
progression stated in Sec. \ref{subsec:sc-LTL}, we develop an augmented
MDP with an LTL instruction $\varphi$, namely the task-driven labeled
MDP (TL--MDP), for LOTUS as follows.
\begin{defn}
\textbf{\label{Def:TL-MDP}}Given a labeled MDP $\mathcal{M}_{e}=\left(S,T,A,p_{e},\Pi,L,R,\gamma,\mu\right)$
corresponding to an LTL task $\varphi$, the TL--MDP is constructed
by augmenting $\mathcal{M}_{e}$ to $\mathcal{M_{\varphi}}\triangleq\left(\tilde{S},\tilde{T},A,\tilde{p},\Pi,L,\tilde{R}_{\varphi},\gamma,\mu\right)$,
where $\tilde{S}=S\times\mathrm{cl}(\varphi)$, $\tilde{T}=\{(s,\varphi)|s\in T,\text{or }\varphi\in\{\mathrm{True},\mathrm{False}\}\},$
$\tilde{p}((s^{'},\varphi^{'})|(s,\varphi),a)=p_{e}(s'|s,a)$ if $\varphi^{'}=\mathrm{prog}(L(s),\varphi)$
and $\tilde{p_{i}}((s^{'},\varphi^{'})|(s,\varphi),a)=0$ otherwise,
and $\tilde{R}_{\varphi}$ is the reward function associated with
the task $\varphi$ to overcome the non-Markovian reward issue which
can be written as
\begin{equation}
\tilde{R}_{\varphi}(s,\varphi)=\begin{cases}
r_{env}+r_{\varphi}, & \text{if \ensuremath{\mathrm{prog}(L(s),\varphi)=\mathrm{True}}}\\
r_{env}-r_{\varphi}, & \text{if \ensuremath{\mathrm{prog}(L(s),\varphi)=\mathrm{False}}}.\\
r_{env}, & otherwise
\end{cases}\label{eq:Markovian Reward}
\end{equation}
where $\mathrm{cl}(\varphi)$ denotes the progression closure of $\varphi$,
$r_{env}$ is the environmental reward, and $r_{\varphi}=\overline{R}(\varphi,\sigma)$
corresponds to the extra reward when $\sigma$ satisfies the LTL task
$\varphi$. Hence, the problem in this work can be stated as follows.
\end{defn}
\begin{problem}
\label{Prob1}Given a labeled TL-MDP $\mathcal{M_{\varphi}}\triangleq\left(\tilde{S},\tilde{T},A,\tilde{p},\Pi,L,\tilde{R}_{\varphi},\gamma,\mu\right)$
corresponding to task $\varphi$ with policy $\pi_{\mathcal{\varphi}}^ {}$,
the goal of this work is to design a universal task representation
model $\varTheta$ parameterized by weights $\xi$, that guides $\pi_{\mathcal{\varphi}}$ toward optimal performance, i.e., the expected return $\underset{\tau\sim\pi_{\varphi}}{\mathbb{E}}\left[\stackrel[k=0]{\infty}{\sum}\gamma^{k}r_{t+k}\mid S_{t}=s\right]$ is maximized.
\end{problem}

\begin{figure}
\centering{}\includegraphics[scale=0.21]{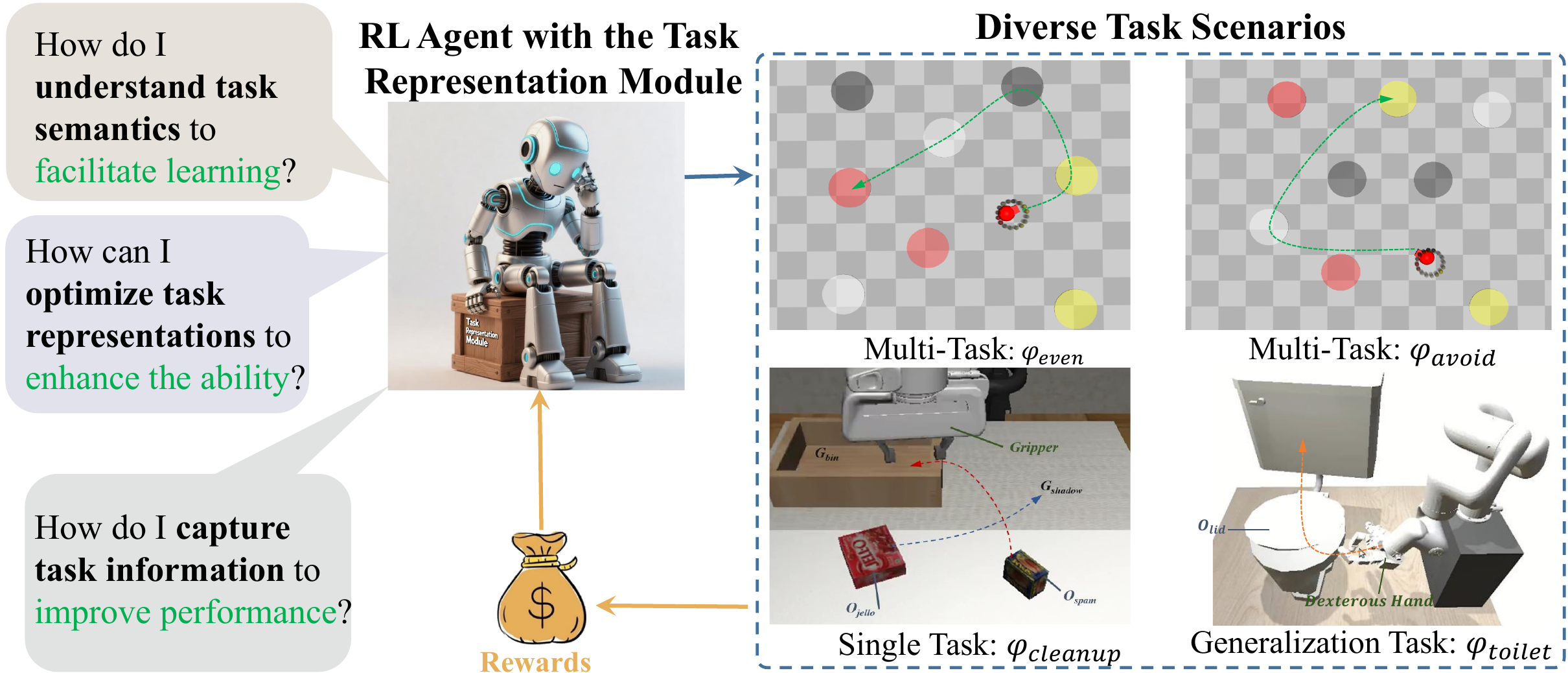}\caption{\textcolor{black}{\label{fig:Outline4Example} Example of a robot
completing different LTL tasks $\varphi_{even}$, $\varphi_{avoid}$,
$\varphi_{cleanup}$, and $\varphi_{toilet}$ across diverse scenarios.
The different colored trajectories present possible solutions for
the robot in completing the tasks.}}
\end{figure}

\begin{table}
\caption{\label{tab:task_in_ex}Task Scenarios and corresponding LTL formulas}

\centering{}\resizebox{0.45\textwidth}{!}{%%
\begin{tabular}{cc}
\hline 
Task Scenarios & \multicolumn{1}{c}{LTL Formulas}\tabularnewline
\hline 
Partially-Ordered Tasks & $\varphi_{even}=\lozenge(\mathrm{Yellow}\wedge\lozenge(\mathrm{Black}\wedge\lozenge(\mathrm{White}\wedge\lozenge\mathrm{Red})))$\tabularnewline
\hline 
Avoidance Tasks & \multicolumn{1}{c}{$\varphi_{avoid}=\varphi_{dang}\cup(\mathrm{White}\wedge(\varphi_{dang}\cup\mathrm{Yellow}))$}
 \tabularnewline
\hline 
Cleanup & $\varphi_{cleanup}=\lozenge(\varphi_{push\text{\_}jellobox}\wedge\lozenge\varphi_{pnp\text{\_}spamcan})$
\tabularnewline
\hline 
Toilet & $\varphi_{toilet}=\lozenge(\mathrm{lid\_approached}\wedge\lozenge(\mathrm{lid\_grasped}\wedge\lozenge\mathrm{lid\_opened}))$\tabularnewline
\hline 
\end{tabular}}
\end{table}

\section{Algorithm Design\label{sec:Algorithm Design}}

In this section, we present a novel task representation framework
for LTL formulas, namely LOTUS, that facilitates the agent's leaning
in diverse task scenarios and updates the task representation more
effectively and stably. We first describe the overview of LOTUS in
Sec. \ref{subsec:Overview} and then explain the technical details
in Sec. \ref{subsec:ltl2graph} to Sec. \ref{subsec:Theoretical_Guarantees}.

\subsection{Overview of LOTUS\label{subsec:Overview}}

In this work, we propose a task representation module that
not only guides the agent to complete tasks efficiently in both
single-task and multi-task scenarios, but also ensures stable representation performance and improves learning efficiency through direct updates.

\begin{figure*}
\centering{}\includegraphics[scale=0.32]{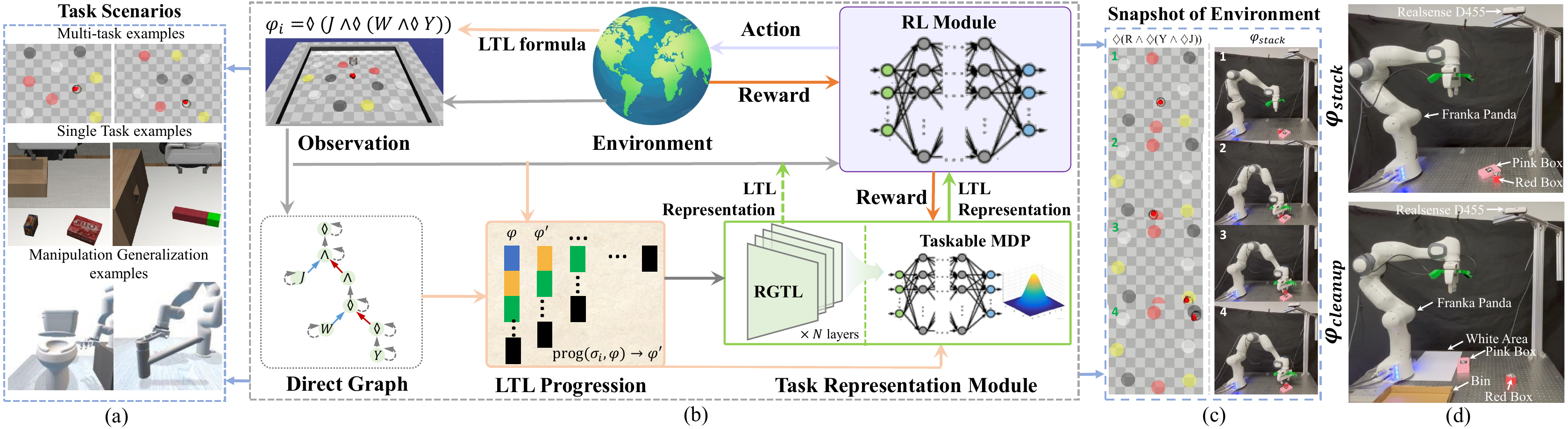}\caption{\textcolor{black}{\label{fig:framework2lotus} The outline of LOTUS.
(a) The task scenarios in our simulation for evaluation.
(b) The architecture of LOTUS. (c) The snapshots of
task execution. (d) The real-world experimental
platform constructed for $\varphi_{stack}$ and $\varphi_{cleanup}$.}}
\end{figure*}

As shown in Fig. \ref{fig:framework2lotus}, we represent the LTL formula as a graph structure to capture the dependencies among different sub-tasks, thereby enabling the agent to better interpret task semantics in both single-task and multi-task settings. Specifically, the proposed RGTN serves as a task representation module that: (1) converts the LTL formula into a directed graph to establish a solid foundation for representation, (2) leverages the relational design of GNNs to model the relevance of different edge types, and (3) employs multi-view attention to effectively integrate global information for generating task representations. Unlike previous methods \cite{vaezipoor2021ltl2action,zhang2023exploiting,yalcinkaya2024compositional}, 
which indirectly updates the task representation by back-propagating
the gradient of the RL controller, our method treats the task representation module as a policy that directly exploits rewards, while ensuring update stability via bisimulation metric.

\begin{algorithm}
\caption{\label{Alg1_LOTUS}LOTUS}

\scriptsize

\singlespacing

\begin{algorithmic}[1]

\Procedure {Input:} {The LTL instruction $\varphi$ and the corresponding
MDP $\mathcal{M}_{e}$}

{Output: } {\textcolor{red}{{} }An Optimized policy $\pi_{\mathcal{\varphi}}^{*}$
of LOTUS in TL-MDP $\mathcal{M}_{\varphi}$}

{Initialization: } {All neural network weights}

\While {episode not terminated}

\For {step $t=0,1,...,T$} \Comment{Exploration Phase}

\State $\varphi^{'}\leftarrow\mathrm{prog}(L(s),\varphi)$

\State Convert the LTL instruction $\varphi^{'}$ to the directed
graph $G_{\varphi^{'}}=\left(\mathcal{V},\mathcal{E},\mathcal{R}\right)$

\State Encode $G_{\varphi^{'}}=\left(\mathcal{V},\mathcal{E},\mathcal{R}\right)$
by (\ref{eq: RGTN}) and sample the corresponding LTL representation
$\hat{\psi}$ from (\ref{eq: sample_ltl_repre})

\State Augment the observation $s$ with $\hat{\psi}$

\State Gather data from $\varphi$ following the policy $\pi_{\varphi}$

\State Add the corresponding transition to the buffer

\EndFor

\For {training step $t=0,1,...,K$} \Comment{Training Phase}

\State Optimized the policy $\pi_{\varphi}$ by its own update way

\State Updated the LTL encoder by (\ref{eq:pi*_bs})

\EndFor

\EndWhile

\State Get Optimized neural network weights

\EndProcedure

\end{algorithmic}
\end{algorithm}

The pseudo-code of LOTUS is shown in Alg. \ref{Alg1_LOTUS}, which
includes an exploration phase (lines 3-10) and a training phase (lines
11-14). During exploration, the current LTL formula is progressed and converted into a directed graph (lines 4-5). The graph is then encoded to obtain an LTL representation (line 6), which is used to augment the observation, guide the agent's interaction with the environment, and store the resulting transitions in the buffer (lines 7-9). In the training phase, transitions are sampled not only to update the RL controller but also to optimize the LTL encoder (lines 12-13). This process repeats until the training budget is exhausted, after which LOTUS terminates and outputs the learned policy along with other network parameters.

In the following development, Sec. \ref{subsec:ltl2graph} describes how tasks specified by LTL instructions can be efficiently transformed into directed graphs, providing a foundation for extracting task semantics. Sec. \ref{subsec:RGTN} introduces the design of RGTN for generating effective LTL representations to guide agents across diverse scenarios. Sec. \ref{subsec:Update_of_LTL_Encoder} details the approach for efficiently and directly updating the LTL encoder based on environmental feedback. Sec. \ref{subsec:LBS2LTL_Encoder} explains how the bisimulation metric is employed to improve the stability of task representations. Finally, Sec. \ref{subsec:Theoretical_Guarantees} presents theoretical guarantees for the proposed task representation module, including behavioral equivalence, optimality fidelity, and trajectory robustness.

\subsection{Convert LTL Specification to Directed Graph\label{subsec:ltl2graph}}

\begin{figure}
\centering{}\includegraphics[scale=0.22]{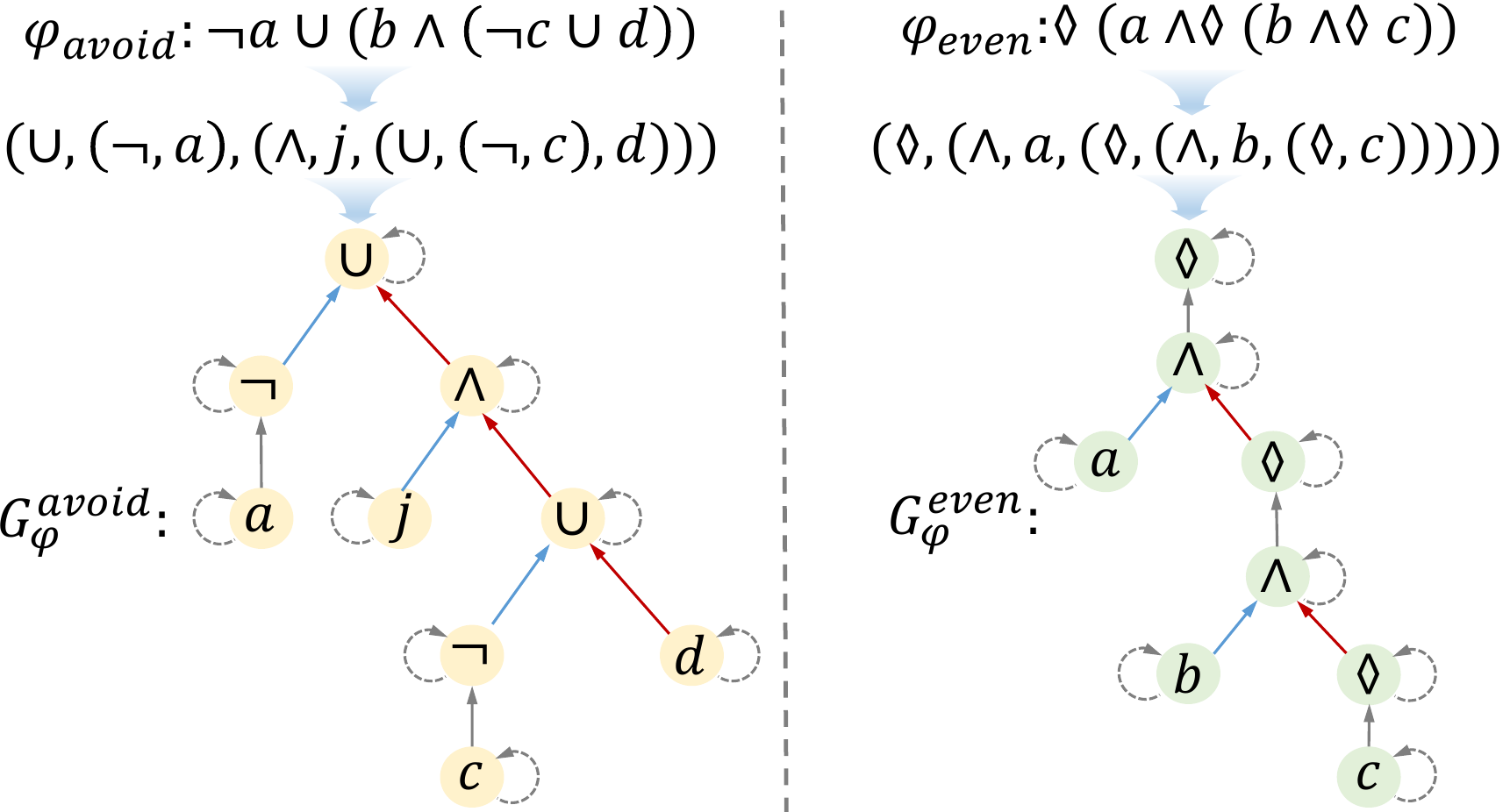}\caption{\textcolor{black}{\label{fig:example_ltl2graph} The examples about
how to build the directed graph from LTL specifications.}}
\end{figure}
There are several ways to model LTL formulas as a basis for guiding agents. For example, LTL formulas can be translated into various forms of automata that constrain robots to accomplish complex tasks in different scenarios \cite{Li2019,Cai2021b,Cai2021c}.
However, using sorted indices to represent automaton states is not well suited for improving RL performance. To address this challenge, LTL tasks are modeled as sequence vectors translated into learnable embedding inputs, a representation that integrates naturally with models such as RNNs and Transformers, and aligns with conventions in human language instruction. Nevertheless, sequence vector modeling is less effective at capturing complex information flows in multi-task settings \cite{vaezipoor2021ltl2action}. Therefore, we exploit the tree-structured nature of LTL formulas and parse them into graphs, which provide stronger inductive biases for generalization and more expressive representation capabilities \cite{si2018learning}.

To make the LTL formula more intuitive for the agent, we first represent
the LTL task $\varphi$ as a directed graph $G_{\varphi}=\left(\mathcal{V},\mathcal{E},\mathcal{R}\right)$,
where $\mathcal{V}$ represents the set of nodes (including the source node
$u$ and the target node $v$), and the type edges $\mathcal{E}$ consists
of a tuple $\left(v,r,u\right)$ where $r\in\mathcal{R}$ represents
edges of different types. In $G_{\varphi}$, the
progressed subtasks are connected to their parent operators via directed
edges, and self-loops are added to all nodes to capture information across different time steps. Specifically, $G_{\varphi}$
has four edge types: (1) \textit{Unary},
connecting child formulas to their parent nodes via a unary operator;
(2) \textit{Binary left}, connecting left
sub-formulas to their parent nodes via a binary operator; (3) \textit{Binary right}, connecting right sub-formulas similarly; (4) \textit{Self-loop}, ensuring that a node's representation can be reused over time. For instance, the tasks $\varphi_{avoid}=\lnot a\cup(b\wedge(\lnot c\cup d))$
and $\varphi_{even}=\diamondsuit(a\wedge\diamondsuit(b\wedge\diamondsuit c))$
can be represented as shown in Fig. \ref{fig:example_ltl2graph}.

\subsection{Encode LTL Representation by RGTN\label{subsec:RGTN}}

One of the main challenges in solving Problem \ref{Prob1} is designing a suitable parameterized encoder for LTL specifications that can effectively guide the agent and improve its performance in single-task, multi-task, and sparse reward scenarios. To address this, LTL2Action \cite{vaezipoor2021ltl2action} employs R-GCN as an LTL encoder for multi-task learning, enabling agents to learn task-conditional policies, but it hardly performs competitively in single-task settings. Conversely, T2TL \cite{zhang2023exploiting} leverages the Transformer to encode LTL representations in single-task situations to boost performance, yet the original Transformer design from \cite{vaswani2017attention} struggles in multi-task scenarios. To overcome these limitations, we propose Relational Graph Transformer Networks (RGTN), which integrate the architectural strengths of both Transformer and R-GCN, allowing agents to achieve better performance across diverse task scenarios.

\begin{figure}
\centering{}\includegraphics[scale=0.29]{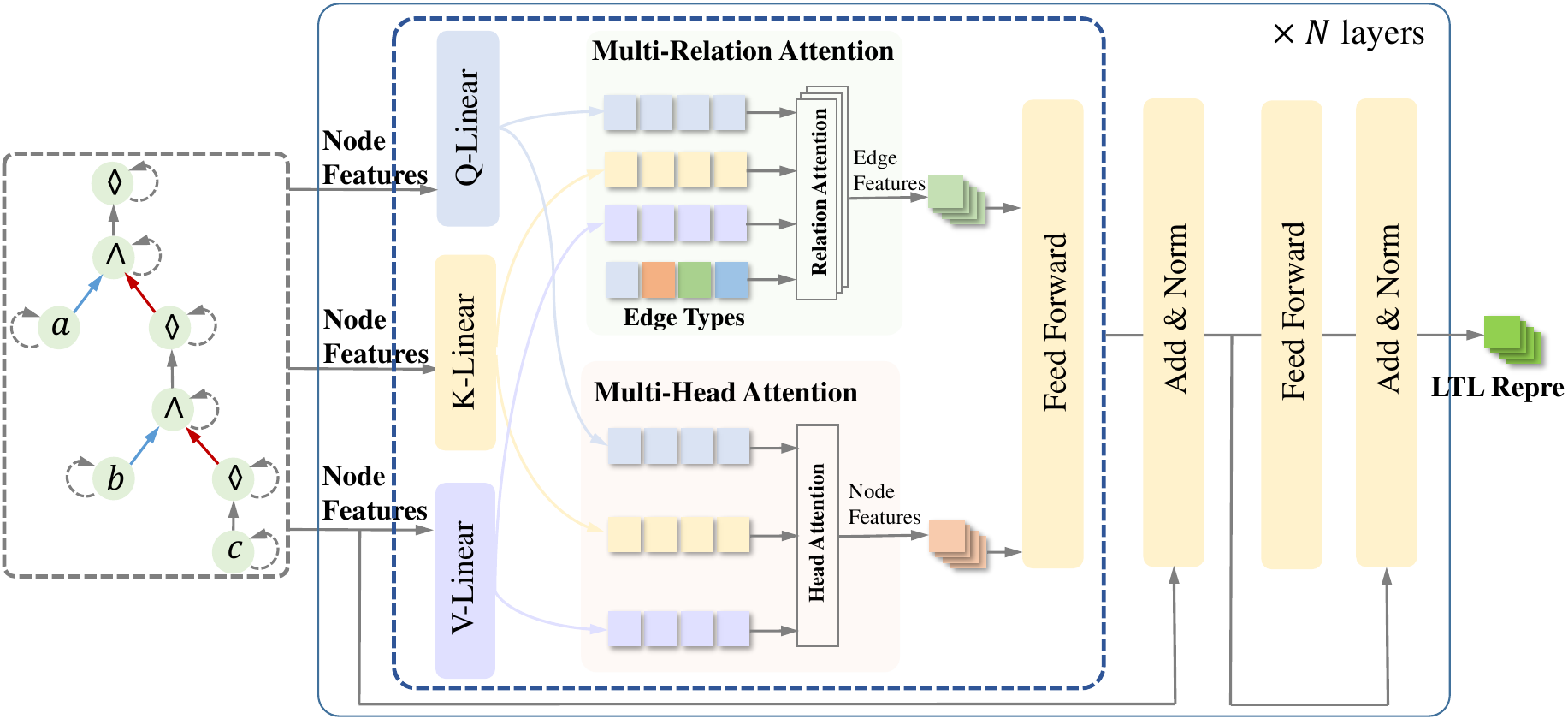}\caption{\textcolor{black}{\label{fig:outline2RGTN} The architecture of RGTN.}}
\end{figure}
As shown in Fig. \ref{fig:outline2RGTN}, the RGTN encoder is composed
of stacked Relational Graph Transformer layers (RGTLs), and each RGTL
is mainly constructed by relational attention sub-layers, head attention
sub-layers, and fully connected feed-forward layers. Layer norm (LN)
is applied to ensure the stability of gradient propagation before
applying each sub-layer and residual connections after each block.
In the structure of RGTN, the self-attention mechanism plays an important
role in capturing the task semantics from different views. Specifically,
given query $Q=W_{Q}h_{0}$, $K=W_{K}h_{0}$, and $V=W_{V}h_{0}$
originating from the directed graph $G_{\varphi}$ containing the
initial node information $h_{0}\in\mathbb{R}^{\left|V\right|\times d_{in}}$
where $W_{Q},W_{K},W_{V}\in\mathbb{R}^{d_{in}\times d_{in}}$ are
learnable projection matrices, we can design two different attention
methods to extract features from the relation edge types view and
global node information view.

1) Multi-Relation Attention (MRA): Given that different types of edges
in a heterogeneous graph contain different node features, we compute
the attention score for each edge relation as 
\[
F_{r}^{\varepsilon}=\mathrm{softmax}\left(\frac{Q_{r}K_{r}^{T}}{\sqrt{d_{k}}}\right)V_{r},
\]
where $Q_{r},K_{r},V_{r}\in\mathbb{R}^{\left|E_{r}\right|\times d_{in}}$ are grouped according to the edge type feature $e_{r}\in\left|E\right|$ and $\left|E_{r}\right|$
denotes the number of edges of relation type $r$. We then integrate the features from different relations to enable the agent to capture diverse task semantics more comprehensively as $m_{i}^{rel}=\underset{\varepsilon=0}{\mathrm{Concat}}F_{r}^{\varepsilon}\in\mathbb{R}^{\left|E\right|\times d_{in}}$.

2) Multi-head Attention (MHA): Following \cite{vaswani2017attention},
we leverage multi-head attention to fuse node features
from different heads, enabling the agent to capture more global context by
$m_{i}^{head}=W_{h}\cdot\underset{\mathfrak{h}=0}{\mathrm{Concat}}F_{h}^{\mathfrak{h}}\in R^{\left|E\right|\times d_{in}}$,
where $F_{h}^{\mathfrak{h}}=\mathrm{softmax}\left(\frac{QK^{T}}{\sqrt{d_{k}}}\right)V.$

(3) Information Fusion: Unlike \cite{vaswani2017attention}, we introduce an additional attention fusion projection  $W_{\mathrm{cat}}$, which further integrates the two attention scores and projects them into a higher-dimensional space by $\tilde{h}_{i}^{l}=W_{\mathrm{cat}}\left[m_{i}^{rel}\parallel m_{i}^{head}\right]$.
Therefore, the main computational procedure of RGTL can be written
as 
\[
h^{l+1}=\sigma\left(W_{0}^{l}h_{i}^{l}+\underset{j\in N(i)}{\sum}\tilde{h}_{j}^{l}\right).
\]
The global computational procedure of RGTN can be summarized as
\begin{equation}
\begin{array}{cc}
h^{1}=\sigma\left(W_{0}^{l}h_{i}^{l}+\underset{j\in N(i)}{\sum}\tilde{h}_{j}^{l}\right), & l=0\\
h_{l}^{'}=\text{LN}(h^{l}))+h^{l}, & l=1,...,L\\
\hat{h}_{l}=\text{MLP}(\text{LN}(h_{l}^{'}))+h_{l}^{'}, & l=1,...,L\\
\hat{\psi}=\text{LN}(\hat{h}_{l}),
\end{array}\label{eq: RGTN}
\end{equation}
where $\hat{\psi}$ stands for the output of the last layer from RGTN,
which can be manually customized to an appropriate dimension according
to the need of tasks.

By designing RGTN, we leverage Multi-Head Attention and residual connections to enable the agent to effectively capture sub-tasks represented by different nodes for improved performance in single-task scenarios. At the same time, the Multi-Relation Attention and directed graph architecture allow the agent to exploit structural information embedded in different edge types, thereby enhancing training efficiency in multi-task settings. However, similar to \cite{vaezipoor2021ltl2action,zhang2023exploiting,yalcinkaya2024compositional}, the update of RGTN depends on the backpropagation of gradients from the RL controller, which slows the convergence of the LTL encoder. In turn, the instability of the LTL representation negatively impacts the stability of the RL controller during training.

\subsection{The Direct Update of LTL Encoder\label{subsec:Update_of_LTL_Encoder}}

Although the design of RGTN can improve the training efficiency in multi-task scenarios, developing a method to directly update the LTL encoder, as stated in Problem \ref{Prob1}, remains an open challenge. Similar to the LTL encoder, the visual encoder may also train unstably and is prone to local optima when its weights are updated only indirectly through backpropagation from the RL controller \cite{levine2016end}. However, unlike the visual encoder, which can be trained independently of the RL module via supervised learning with human-labeled data \cite{bao2023dexart}, there is no readily available ground-truth LTL representation to enable supervised training of the LTL encoder.

To address the above challenge, we propose to model
the LTL encoder as a stochastic
policy inspired by \cite{zhang2024exploiting}. This allows task representations to be sampled directly from the policy and optimized using the policy gradient method. We refer to this approach as TLPG, which consists of two key
steps. In the first step, we model the interaction with LTL Progression
at the task level as a taskable MDP.

\begin{defn}
\label{def:Taskable_MDP}(Taskable MDP) A taskable MDP is defined
as a tuple $\mathcal{M}_{\Phi}=\left(\Psi,\Gamma,\varXi,p_{\Phi},\Pi,\tilde{R}_{\varphi},\gamma_{\Phi},\mu_{\Phi}\right)$,
where $\Psi$ is the state space containing the LTL formulas, $\Gamma\subseteq\Psi$
is the set of terminal LTL formulas, $\varXi$ is the action space
composed by LTL representations, $\Pi$ is a set of atomic propositions,
$p_{\Phi}(\varphi^{'}|\varphi,\sigma)=\mathrm{prog}(\sigma,\varphi)$
is the transition function from $\varphi\in\Psi$ to $\varphi^{'}\in\Psi$
under the proposition $\sigma\in\Pi$, $\tilde{R}_{\varphi}$ is the
reward function in Def. \ref{Def:TL-MDP}, $\gamma_{\Phi}\in(0,1]$
is the discount factor, and $\mu_{\Phi}$ is the initial LTL formula
distribution.
\end{defn}

Def. \ref{def:Taskable_MDP} provides several advantages.
First, the LTL encoder is treated as a policy in Taskable MDP,
interacting at the task level and producing the LTL representation $\hat{\psi}$. This representation guides the interactions between the RL agent and the environment, yielding corresponding propositions via $\sigma=L(s)$. Consequently,
task-level rewards are obtained whenever the LTL specification progresses.
More concretely, we can model the RGTN as a stochastic strategy with
a Gaussian distribution 
\begin{equation}
\pi_{\zeta}^{\varTheta}=\mathcal{N}(\hat{\psi},\varrho),\label{eq: sample_ltl_repre}
\end{equation}
where $\zeta$ denotes the policy weights, $\hat{\psi}$ is the LTL representation encoded by RGTN, and $\varrho$ is a learnable variance parameter.
Thus a new LTL representation $\hat{\psi}_{\pi}$ can be sampled from
$\pi_{\zeta}^{\varTheta}$. Second, by modeling the LTL encoder
as a stochastic policy in Taskable MDP, we can optimize the policy weights using an on-policy method.
Therefore, the second key step in TLPG is to directly optimize
the LTL encoder by designing the following loss function as
\begin{equation}
J_{\pi^{\varTheta}}\left(\zeta\right)=\underset{\tau\sim\pi_{\zeta}^{\varTheta}}{E}\left[-\log\pi_{\zeta}^{\varTheta}\cdot\tilde{R}_{\varphi}\right].\label{eq: TLPG}
\end{equation}
It is worth noting that we use reward $\tilde{R}_{\varphi}$ in (\ref{eq: TLPG})
because it can incorporate not only the task reward $r_{\varphi}$ in sparse reward settings, but also the environmental reward $r_{env}$ in
dense reward settings, thereby accelerating the optimization of LTL encoder. For instance, the dense rewards provided by DexArt in Sec. \ref{sec:EXPERIMENTS}
can enable agents to achieve higher success rates.
Thus when the task representation module is updated, it leverages both the task-level reward $r_{\varphi}$ and the environment-level reward
$r_{env}$ to continually refine the LTL representation.

\subsection{Optimizing LTL Encoder by Bisimulation Metric\label{subsec:LBS2LTL_Encoder}}

In Sec. \ref{subsec:Update_of_LTL_Encoder}, we present a direct update for the LTL encoder by designing a stochastic strategy using environmental feedback. However, the challenge of further improving the robustness and stability of the representation quality in Problem \ref{Prob1} remains unresolved. \textcolor{black}{Potential drawbacks include: 1) Compared
to RGTN, TLPG samples LTL representations $\hat{\psi}_{\pi}$ from
the stochastic policy $\pi_{\zeta}^{\varTheta}$; sampling noise may
introduce representation bias, thereby impairing the agent's performance.
2) In dense reward environments (e.g., manipulation tasks), the sampled
representations are prone to reward fluctuations, compromising representational
stability.} To address this challenge, inspired by \cite{ferns2011bisimulation,castro2020scalable,zhang2020learning},
we utilize the bisimulation metric introduced in Sec. \ref{subsec:BS_in_RL} to enhance stability and robustness during updates.
Specifically, we introduce three key design steps to strengthen TLPG through the bisimulation metric.

Similar to Def. \ref{def:bs_in_rl}, we define the bisimulation
metric within Taskable MDP to evaluate the similarity between two LTL formulas $\phi_{i}$ and $\phi_{j}$ as
\begin{equation}
\begin{array}{cc}
d\left(\phi_{i},\phi_{j}\right)= & \underset{\sigma\in\Pi}{\max}\left(1-c\right)\left|\overline{R}\left(\phi_{i},\sigma\right)-\overline{R}\left(\phi_{j},\sigma\right)\right|\\
 & +cW_{1}\left(p_{\Phi}\left(\cdot\mid\phi_{i},\sigma\right),p_{\Phi}\left(\cdot\mid\phi_{j},\sigma\right)\right),
\end{array}\label{eq:bs_2011}
\end{equation}
where $c\in(0,1)$ is a weighting coefficient that balances the contributions of reward differences and transition distribution differences. A smaller value of $d\left(\phi_{i},\phi_{j}\right)$ indicates greater similarity between $\phi_{i}$ and $\phi_{j}$. However, computing the max operator in (\ref{eq:bs_2011}) is challenging
in high-dimensional spaces. 

Inspired by \cite{castro2020scalable},
we simplify the computation of the bisimulation metric in
(\ref{eq:bs_2011}) by considering the policy $\pi_{\zeta}^{\varTheta}$ introduced in Sec. \ref{subsec:Update_of_LTL_Encoder}:
\begin{equation}
d\left(\phi_{i},\phi_{j}\right)=\left|r_{\phi_{i}}^{\pi}-r_{\phi_{j}}^{\pi}\right|+\gamma W_{1}\left(P^{\pi}\left(\cdot\mid\phi_{i}\right),P^{\pi}\left(\cdot\mid\phi_{j}\right)\right)\label{eq:on_policy_bs_in_ltl}
\end{equation}
where $r_{\phi}^{\pi}=\underset{\sigma}{\sum}\pi_{\zeta}^{\varTheta}\left(\hat{\psi}_{\pi}|\phi\right)\overline{R}(\phi,\sigma)$
and $P^{\pi}\left(\cdot\mid\phi\right)=\underset{\sigma}{\sum}\pi_{\zeta}^{\varTheta}\left(\hat{\psi}_{\pi}|\phi\right)\underset{\phi^{'}\in\varPhi}{\sum}p_{\Phi}\left(\phi^{'}|\phi,\sigma\right)$.
Although the metric in (\ref{eq:on_policy_bs_in_ltl}) effectively measures the similarity between $\phi_{i}$
and $\phi_{j}$, our ultimate goal is to obtain more stable representations through the encoder output towards desired realation $d\left(\phi_{i},\phi_{j}\right)$. To this end, we enforce a comparable scale in the latent space and further optimize the LTL encoder via
\begin{equation}
\begin{array}{cc}
J\left(\varTheta\right)= & \left(\left\Vert \hat{\psi}_{\pi}^{i}-\hat{\psi}_{\pi}^{j}\right\Vert _{1}-\left|r_{\phi_{i}}^{\pi}-r_{\phi_{j}}^{\pi}\right|-\right.\\
 & \left.\gamma W_{2}\left(\hat{P}\left(\cdot\mid\hat{\psi}_{\pi}^{i},a_{i}\right),\hat{P}\left(\cdot\mid\hat{\psi}_{\pi}^{j},a_{j}\right)\right)\right)^{2}
\end{array}\label{eq:pi*_bs}
\end{equation}
where $\hat{\psi}_{\pi}^{i}$ and $\hat{\psi}_{\pi}^{j}$ are sampled
from $\pi_{\zeta}^{\varTheta}$, $\hat{P}$ denotes the latent dynamics model parameterized as a Gaussian distribution, and the 2-Wasserstein distance $W_{2}$ is adopted due to its closed-form expression for Gaussian distribution. 

\textcolor{black}{Therefore, by Alg. \ref{Alg1_LOTUS},
the algorithmic flow of LOTUS can be summarized as follows. At the
task initiation stage, the agent receives an LTL task $\varphi$,
which is first converted into a directed graph $G_{\varphi}$ to provide
stronger inductive bias. Subsequently, the directed graph $G_{\varphi}$
is input to the LTL encoder RGTN via (2) to extract task semantics
$\hat{\psi}$, thereby guiding the agent-environment interaction.
As the agent continuously explores and exploits, sub-goals are progressively
satisfied. Their corresponding atomic propositions $\sigma_{i}$ are
then input at the task level via $\mathrm{prog}(\sigma_{i},\varphi)$
to gradually shorten the LTL task, until the task is fulfilled, fails,
or exceeds the specified number of steps. \textcolor{black}{To enable direct and stable updates
to RGTN using rewards, we treat the state transitions achieved at
the task level via $\mathrm{prog}$ as a Taskable MDP. Within this
MDP, we model RGTN as a stochastic policy, using its sampled LTL representations
as actions to guide the agent in achieving state transitions within
the TL-MDP, and drives state transitions at the task level (e.g.,
task progression) simultaneously. Specifically, the sc-LTL task samples
the corresponding representation $\hat{\psi}$ via (2) and (3), then
concatenates the representation and observations together to serve
as input for the downstream RL agent, guiding its actions and environmental
sampling to obtain the corresponding reward value and the next-time-step
state. We then use a labeling function to map this environmental state
to corresponding atomic propositions and utilize LTL Progression to
achieve state transitions at the task level within the Taskble MDP.
This enables both task-level progression and environmental-level interaction.
However, as mentioned earlier, directly using task representations
sampled from a stochastic policy to guide the agent introduces instability
in updates from multiple perspectives. Therefore, we employ a bisimulation
metric approach to further constrain the task representations encoded
by the LTL module. This aims to enhance the stability and robustness
of task representations by quantifying behavioral similarity between
task formulas, thereby improving the agent\textquoteright s training
performance and generalization capabilities in downstream environments. The comparison of different
update methods for the LTL encoder is shown in Fig. \ref{fig:Update_TLPG}.}}

\textcolor{black}{Since our task representation framework is applicable
to most RL algorithms, suppose the time complexity of the RL forward
pass is $O(T\cdot C_{RL}^{for})$ and the time complexity during updates
is $O(K\cdot C_{RL}^{up})$.} \textcolor{black}{The overall time complexity
of Alg. \ref{Alg1_LOTUS} is $O\left(TL\left(E|d^{2}+H|V|^{2}d\right)+K|B|d^{2}+TC_{RL}^{for}+KC_{RL}^{up}\right)$}.\textcolor{black}{{}
The computational complexity is dominated by the self-attention mechanism
in the RGTN, which scales with a time complexity of $O\left(T\cdot|V|\text{\ensuremath{\cdot}}d^{2}\right)$.
The space complexity is mainly composed of LTL graph storage, LTL
encoder feature storage, and replay buffer occupancy, and can be summarized
as $O\left(L|V|d+|D|d\right)$}.

\begin{rem}
\textcolor{black}{The scope of LOTUS\textquoteright s universality
generally encompasses long-horizon tasks described by sc-LTL, as well
as the corresponding MDP comprising the task environment and the agent,
which require the policy to be differentiable. It is worth noting
that, since LOTUS is dedicated to encoding task semantics at the task
level to guide the agent in the environment, there are no restrictions
on whether the environment in which the downstream agents play is
discrete or continuous, nor on whether the RL algorithms driving the
downstream agents are on-policy or off-policy.}
\end{rem}

\subsection{Theoretical Guarantees of LTL Encoder\label{subsec:Theoretical_Guarantees}}

In this section, we show the theoretical guarantees of LTL representations
with bisimulation metric, including behavioral equivalence, optimality
fidelity, and trajectory robustness. 

To facilitate the following analysis, we use $P_{\phi_{\boxempty}}^{\pi}$
to denote the corresponding $P^{\pi}\left(\cdot\mid\phi_{i}\right)$
(e.g., $\boxempty$ can take $i$, $j$), leverage $V_{\boxempty}^{*}$
to denote the corresponding $V^{*}(\cdot)$ with different input (e.g.,
$\boxempty$ can take $\phi$), and leverage $R_{\diamond}^{\boxempty}$
to denote the corresponding $R(\cdot,\cdot)$ with different state-action
pairs (e.g., $\boxempty$ can take $s$ and $\diamond$ can take $a$).
\begin{thm}
\label{thm:BE_bs}(Behavioral Equivalence for $\mathrm{LOTUS}$) If a policy $\pi$ in $\mathrm{LOTUS}$ improves continuously over time and converges to the optimal policy $\pi^{*}$, the following bisimulation metric between two LTL formula $\phi_{i}$ and $\phi_{j}$ in task state space
\[
\begin{array}{cc}
d\left(\phi_{i},\phi_{j}\right)= & \left(1-c\right)\left|r_{\phi_{i}}^{\pi}-r_{\phi_{j}}^{\pi}\right|+c\left(\inf_{\gamma^{'}\in\varGamma\left(P_{i},P_{j}\right)}\right.\\
 & \left.\int_{\varphi\times\varphi}d\left(\phi_{i},\phi_{j}\right)^{p}d\gamma^{'}\left(\phi_{i},\phi_{j}\right)\right)^{1/p}
\end{array}
\]
has a least fixed point
which is a $\pi^{*}$-bisimulation metric, where $\varGamma\left(P_{i},P_{j}\right)$ is the set of all couplings
of $P_{i}$ and $P_{j}$.
\end{thm}
The proof of Thm. \ref{thm:BE_bs} is omitted, as it is a straightforward extension of the Thm. 1 in \cite{zhang2020learning}.
Under this $\pi^{*}$-bisimulation metric, we can divide the latent
space into $n$ partitions based on some $\epsilon>0$, where $\frac{1}{n}<(1-c)\epsilon$.
With these notations, the following value functions can be bounded based
on bisimulation metrics.
\begin{thm}
\label{thm:OF_bs}(Optimality Fidelity for $\mathrm{LOTUS}$) Consider
a taskable MDP $\mathcal{\overline{M}}_{\Phi}$, which is formed
by clustering states within an $\epsilon$-neighborhood, along with
an encoder $\varTheta$ that maps states from the original MDP $\mathcal{M}_{\Phi}$
to these clusters. Under the same assumption in Thm. \ref{thm:BE_bs},
the optimal value functions for the two MDPs are bounded by
\[
\left|V_{\phi}^{*}-V_{\varTheta\left(\phi\right)}^{*}\right|\leq\frac{2\epsilon+2\mathcal{L}}{\left(1-\gamma\right)\left(1-c\right)}
\]
where $\mathcal{L}=\sup\left|\left\Vert \varTheta\left(\phi_{i}\right)-\varTheta\left(\phi_{j}\right)\right\Vert -d\left(\phi_{i},\phi_{j}\right)\right|$
is the learning error for LTL encoder $\varTheta$. 
\end{thm}
\begin{IEEEproof}
From Thm. 5.1 in \cite{ferns2004metrics} we have
\begin{equation}
\left(1-c\right)\left|V_{\phi}^{*}-V_{\varTheta\left(\phi\right)}^{*}\right|<g(\phi,d)+\frac{\gamma}{1-\gamma}\underset{\rho\in\phi}{\max}g\left(\rho,d\right)\label{eq: thm5.1 in 2004}
\end{equation}
where $g$ describes the average distance between a state and all
other states in its equivalence class under the bisimulation metric
$d$. By specifying $\epsilon$-neighborhood for each cluster of states
we can replace $g$ in (\ref{eq: thm5.1 in 2004}) as follows 
\[
\begin{array}{ccc}
\left(1-c\right)\left|V_{\phi}^{*}-V_{\varTheta\left(\phi\right)}^{*}\right| & \leq & 2\epsilon+\frac{\gamma}{1-\gamma}2\epsilon\\
\left|V_{\phi}^{*}-V_{\varTheta\left(\phi\right)}^{*}\right| & \leq & \frac{1}{1-c}\left(2\epsilon+\frac{\gamma}{1-\gamma}2\epsilon\right)\\
 & = & \frac{2\epsilon}{\left(1-\gamma\right)\left(1-c\right)}\\
 & \leq & \frac{2\epsilon+2\mathcal{L}}{\left(1-\gamma\right)\left(1-c\right)}
\end{array}
\]

\end{IEEEproof}

\begin{figure}
\centering{}\includegraphics[scale=0.26]{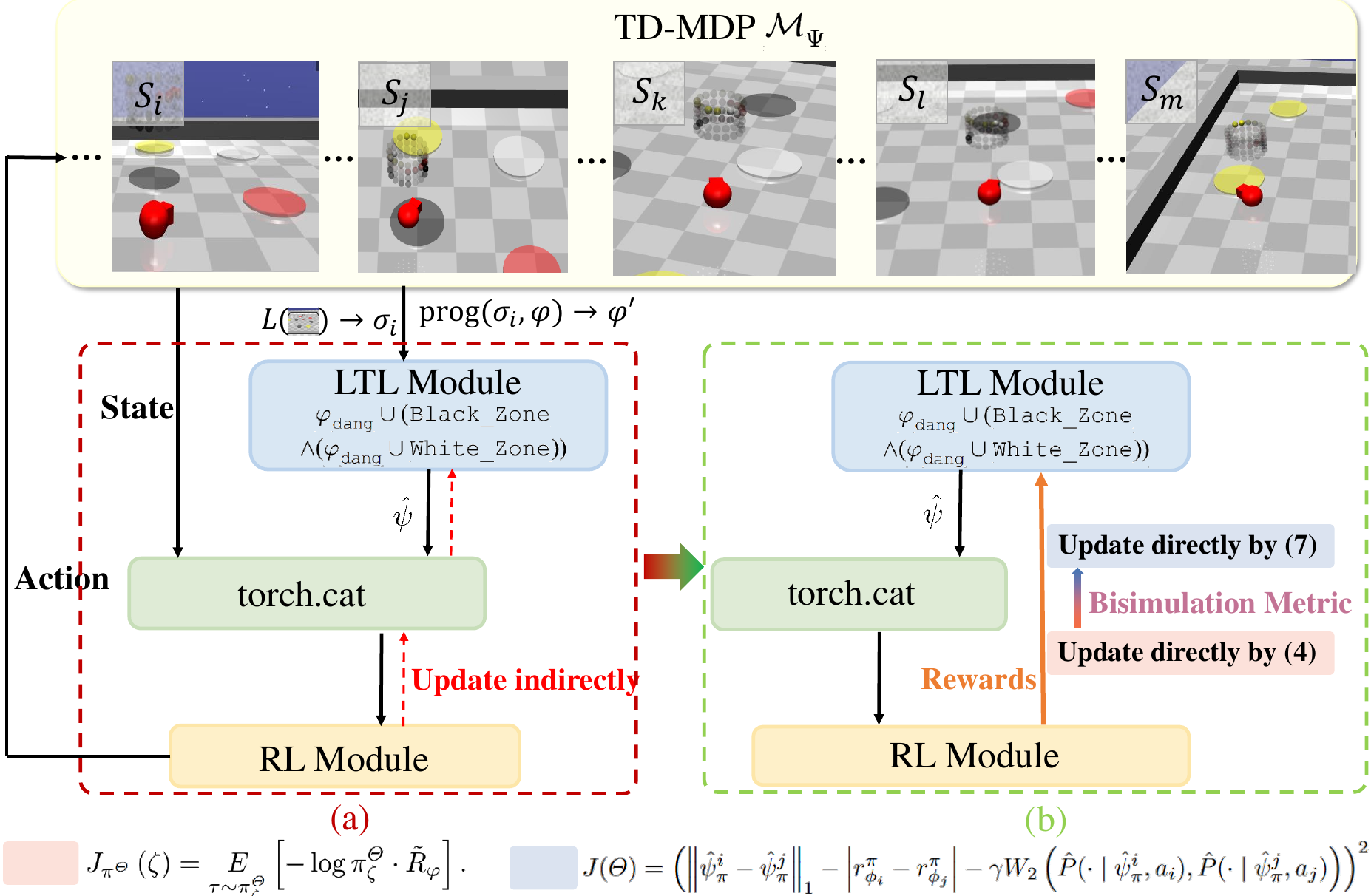}\caption{\textcolor{black}{\label{fig:Update_TLPG} The }comparison of different
update methods for the LTL encoder. (a) Optimize the LTL representations
by the back- from the RL controller. (b) Regard the LTL encoder as
a policy and further optimize it by leveraging the rewards from the
environment.}
\end{figure}

Thm. \ref{thm:OF_bs} demonstrates that the optimal value function
in the original state space and the optimal value function in the
latent space, projected by the $\pi^{*}$-bisimulation metric, is
bounded from above. Leveraging Thm. \ref{thm:BE_bs} and Thm. \ref{thm:OF_bs},
we can bound the cumulative reward of a trajectory under the original
MDP $\mathcal{M}_{\Phi}$ and the latent taskable MDP $\mathcal{\overline{M}}_{\Phi}$.
\begin{thm}
\label{thm:TR_bs}(Trajectory Robustness for $\mathrm{LOTUS}$) Consider
a trajectory $\tau=(\phi_{0},\sigma_{0},\phi_{1},\sigma_{1},...,\phi_{H-1},\sigma_{H-1},\phi_{H})$
in the original state space $\Psi$, and its corresponding encoded
trajectory $\varTheta(\tau)=(\hat{\psi}_{0},\sigma_{0},\hat{\psi}_{1},\sigma_{1},...,\hat{\psi}_{H-1},\sigma_{H-1},\hat{\psi}_{H})$,
where $\hat{\psi}_{K}=\varTheta(\phi_{k})$, with $\varTheta(\cdot)$
defined as in Thm . \ref{thm:OF_bs}. Under the same assumption as
in Thm. \ref{thm:BE_bs} and Thm. \ref{thm:OF_bs} , the following
expected cumulative rewards 
\[
\varPhi\left(\tau\right)=\mathbb{E}_{\tau}\left[\gamma^{H}V_{\phi_{H}}^{*}+\stackrel[h=0]{H-1}{\sum}\gamma^{h}R_{a_{h}}^{\phi_{h}}\right],
\]
\[
\varPhi\left(\varTheta\left(\tau\right)\right)=\mathbb{E}_{\tau}\left[\gamma^{H}V_{\varTheta\left(\phi_{H}\right)}^{*}+\stackrel[h=0]{H-1}{\sum}\gamma^{h}R_{a_{h}}^{\varTheta\left(\phi_{h}\right)}\right]
\]
can be bounded as follows 
\[
\left|\varPhi\left(\tau\right)-\varPhi\left(\varTheta\left(\tau\right)\right)\right|\leq\frac{2\gamma^{H}\left(\epsilon+\mathcal{L}\right)}{\left(1-\gamma\right)\left(1-c\right)}+\frac{2\epsilon\left(1-\gamma^{H}\right)}{\left(1-\gamma\right)\left(1-c\right)}.
\]
\end{thm}
\begin{IEEEproof}
We can calculating the difference between $\varPhi\left(\tau\right)$
and $\varPhi\left(\varTheta\left(\tau\right)\right)$to prove the
boundary as follows.
\[
\begin{array}{cc}
 & \left|\varPhi\left(\tau\right)-\varPhi\left(\varTheta\left(\tau\right)\right)\right|\\
= & \left|\mathbb{E}_{\tau}\left[\gamma^{H}\left(V_{\phi}^{*}-V_{\varTheta\left(\phi\right)}^{*}\right)+\stackrel[h=0]{H-1}{\sum}\gamma^{h}\left(R_{a_{h}}^{\phi_{h}}-R_{a_{h}}^{\varTheta\left(\phi_{h}\right)}\right)\right]\right|\\
\leq & \gamma^{H}\mathbb{E}_{\tau}\left[\left|V_{\phi}^{*}-V_{\varTheta\left(\phi\right)}^{*}\right|\right]+\stackrel[h=0]{H-1}{\sum}\gamma^{h}\mathbb{E}_{\tau}\left[\left|R_{a_{h}}^{\phi_{h}}-R_{a_{h}}^{\varTheta\left(\phi_{h}\right)}\right|\right]\\
\leq & \frac{2\gamma^{H}\left(\epsilon+\mathcal{L}\right)}{\left(1-\gamma\right)\left(1-c\right)}+\frac{2\epsilon}{\left(1-c\right)}\stackrel[h=0]{H-1}{\sum}\gamma^{h}\\
\leq & \frac{2\gamma^{H}\left(\epsilon+\mathcal{L}\right)}{\left(1-\gamma\right)\left(1-c\right)}+\frac{2\epsilon\left(1-\gamma^{H}\right)}{\left(1-\gamma\right)\left(1-c\right)}
\end{array}
\]

\end{IEEEproof}
Thm. \ref{thm:TR_bs} means that if the cluster radius $\epsilon$
and the encoder error $\mathcal{L}$ are sufficiently small, the learned
representation space does not change the original cumulative rewards
over the same trajectory $\tau$. This suggests that the latent space
retains essential information from the original space.

\section{EXPERIMENTS\label{sec:EXPERIMENTS}}

In this section, the developed LOTUS framework is evaluated against
the state-of-the-art algorithms in simulation and real world. In particular,
we consider the following aspects. \textit{1) Performance:} whether
RGTN can effectively guide the agent in both single-task and multi-task
scenarios. \textit{2) Encoding capability:} how RGTN improves the
performance of the agent compared to other encoders across different
task scenarios. \textit{3) Task Generalization:} analyzing the generalization
performance of the agent under the guidance of LOTUS. \textit{4) Efficiency:}
evaluating the impact of different roles in LOTUS for policy improvement.

\subsection{Experimental Setup}

\textit{1) Environments:} To validate the effectiveness of LOTUS across different task scenarios, we select following four
types of environments as the evaluation benchmark to compare with other baselines.
The first environment is \textit{Robosuite} \cite{nasiriany2022augmenting},
in which four single-task scenarios with dense rewards
feedback are selected for performance comparison, including Stack, Nut Assembly,
Cleanup, and Peg Insertion. The second and third environments are \textit{LetterWorld}
and \textit{ZoneEnv} modified from \cite{vaezipoor2021ltl2action}
with sparse reward setting. The former provides a series of atomic
propositions in discrete space that can effectively validate the algorithm's
guidance in long-horizon tasks, while the latter presents a challenging
environment for validating the algorithm's representation in continuous
environments. The fourth environment is \textit{DexArt} \cite{bao2023dexart},
which differs from \cite{nasiriany2022augmenting} in that it involves
dexterous manipulation on four articulated objects with dense rewards,
including Toilet, Trashcan, Faucet, and Dispenser, and can
evaluate the impact of different algorithmic guidance on the agent's
generalized performance.

\begin{figure*}
\centering{} \includegraphics[scale=0.335]{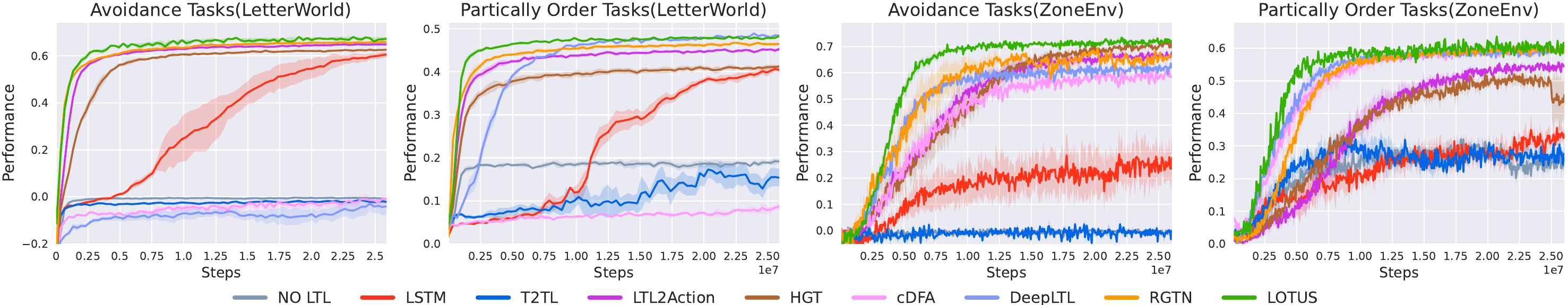}\caption{\textcolor{black}{\label{fig:plot_reward4multi} Plots of discounted
return curves for four kinds of multi-task learning from different
algorithms in two scenarios.}}
\end{figure*}

\begin{figure*}
\centering{} \includegraphics[scale=0.335]{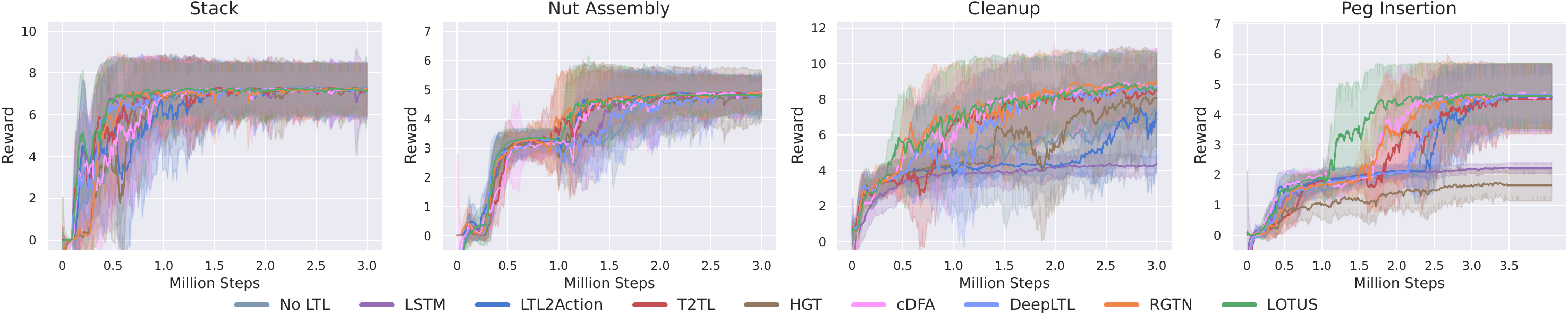}\caption{\textcolor{black}{\label{fig:plot_reward2Single} Plots of normalized
reward curves from different algorithms in four manipulation tasks.}}
\end{figure*}

\textit{2) Baselines:} To demonstrate the effectiveness of the LOTUS
framework, we empirically compare it against the following baselines. The
first baseline is DFA \cite{Lacerda2014}, which constructs
the product MDP for LTL tasks over a finite horizon. The second
baseline is LSTM \cite{Kuo2020}, which encodes sequence vectors
interpreted by LTL specifications to enable zero-shot
generalization. The third baseline is T2TL \cite{zhang2023exploiting},
which performs well in single task scenarios by leveraging an LTL representation encoded by Transformer. The fourth baseline is LTL2Action \cite{vaezipoor2021ltl2action}, which employs an R-GCN to represent LTL instructions in multi-task settings and guides the agent in learning task-conditioned policies.
The fifth baseline is Heterogeneous Graph Transformer (HGT) \cite{hu2020heterogeneous}, which extends R-GCN with Transformer architecture to better capture features from heterogeneous graphs. \textcolor{black}{The sixth baseline is cDFA from \cite{yalcinkaya2024compositional}
which encodes temporal goals by GAT \cite{brody2021attentive} and
leverages cDFA to guide the RL agent.} \textcolor{black}{The seventh
baseline is DEEPLTL from \cite{deepltl} , which explicitly represents the
semantics of LTL specifications encoded from the $\mathrm{B\ddot{u}chi}$
automata to learn policies conditioned on sequences of truth assignments
that lead to satisfying the desired formulae. }The eighth baseline is our proposed RGTN in Sec. \ref{subsec:RGTN} to validate its representation ability. Finally, we add a No LTL baseline, where the agent operates without LTL guidance, to highlight the importance of task representations.

\subsection{Main Performance in Diverse Task Scenarios}

\textit{1) Multi-task Performance.} We first present the reward
curves of different algorithms in multi-task environments, including LetterWorld and ZoneEnv, as shown in Fig. \ref{fig:plot_reward4multi}. Several
observations can be made. \textcolor{black}{(1) Most of methods guided
by the LTL representation can achieve better task performance compared
to the algorithm without LTL representation. This validates that extracting
task semantics from LTL formulas effectively mitigates exploration
inefficiency and goal ambiguity in multi-task learning. (2) Modeling
LTL formulas as tree structures (R-GCN, HGT, cDFA, RGTN and LOTUS)
is more advantageous than sequence vectors (LSTM and Transformer)
when combined with LTL encoders to enhance their representational
capabilities in multi-task scenarios. Tree structures naturally capture
the hierarchical relationships and associative properties inherent
in LTL formulas, whereas sequence vectors flatten temporal logic into
linear sequences, losing cross-task structural invariance. (3) Algorithms
that only consider edge-relative information (R-GCN) or only involve
attention computation (Transformer and cDFA) perform worse than those
(HGT, RGTN and LOTUS) that integrate both. R-GCN can only model local
edge relationships and struggles to capture global task structures.
Transformer and cDFA rely on attention to obtain global information
but lack explicit relationship modeling and exhibit unstable convergence
in multi-task scenarios. (4) DeepLTL typically performs better in
partially-ordered tasks, as the absence of obstacles in such scenarios
minimizes significant disruptions to curriculum learning. (5) Unlike
HGT, which processes node features on edges, RGTN better guides agent
actions by further integrating local and global information. In Avoidance
task of ZoneEnv with increasing conjunctions, the return of HGT drop
from 0.992 under i.i.d settings to 0.576, while RGTN maintains a higher
reward of 0.719. (6) Thanks to TLPG and bisimulation metric, LOTUS
demonstrates faster convergence trends and more stable training performance
across all environments. By eliminating the dependency of encoder
updates on RL controller gradients, LOTUS accelerates task convergence
compared to indirect update methods like LTL2Action. Simultaneously,
the bisimulation metric mitigates representation instability.}

\textit{2) Single-task Performance.} We then present the training
performance of different algorithms in four Robosuite environments 
\cite{nasiriany2022augmenting}. From Fig. \ref{fig:plot_reward2Single},
we can observe the following interesting results. (1) Under the guidance of LTL representation (Transformer, HGT, RGTN, and LOTUS), the training efficiency of the baseline
manipulation algorithm (No LTL) is notably improved. (2) Although modeling LTL formulas as sequence vectors performs poorly in multi-task scenarios, it demonstrates competitive performance in single-task settings (Transformer).
(3) LOTUS and RGTN consistently outperform other algorithms, with LOTUS showing a clear advantage in the challenging peg insertion task, converging approximately 2500 episodes earlier than LSTM.

\subsection{Generalization Evaluation}

\begin{table}
\caption{\label{tab:Genal2ZoneEnv}Total return of different methods
in two scenarios (LetterWorld and ZoneEnv) with different task types.}

\centering{}\resizebox{0.41\textwidth}{!}{%%
\begin{tabular}{c|c|c|c}
\hline 
\textcolor{black}{Task} & \multicolumn{3}{c}{\textbf{\textcolor{black}{LetterEnv}}}\tabularnewline
\hline 
\textcolor{black}{Type} & \multicolumn{3}{c}{\textcolor{black}{Avoidance Tasks}}\tabularnewline
\hline 
\textcolor{black}{Split} & \textcolor{black}{I.I.D} & \textcolor{black}{Depth ($\uparrow$)} & \textcolor{black}{Conjuncts ($\uparrow$)}\tabularnewline
\hline 
\textcolor{black}{No LTL} & \textcolor{black}{-0.054 $\pm$ 0.228} & \textcolor{black}{-0.054 $\pm$ 0.229} & \textcolor{black}{-0.072 $\pm$ 0.248}\tabularnewline
\textcolor{black}{LSTM} & \textcolor{black}{0.883 $\pm$ 0.246} & \textcolor{black}{0.464 $\pm$ 0.551} & \textcolor{black}{0.164 $\pm$ 0.406}\tabularnewline
\textcolor{black}{T1TL} & \textcolor{black}{-0.071 $\pm$ 0.344} & \textcolor{black}{-0.086 $\pm$ 0.278} & \textcolor{black}{-0.223 $\pm$ 0.422}\tabularnewline
\textcolor{black}{LTL2Action} & \textcolor{black}{0.965 $\pm$ 0.233} & \textcolor{black}{0.899 $\pm$ 0.314} & \textcolor{black}{0.536 + 0.557}\tabularnewline
\textcolor{black}{HGT} & \textcolor{black}{0.963 $\pm$ 0.251} & \textcolor{black}{0.783 $\pm$ 0.424} & \textcolor{black}{0.487 $\pm$ 0.593}\tabularnewline
\textcolor{black}{cDFAs} & \textcolor{black}{-0.050 $\pm$ 0.363} & \textcolor{black}{-0.141 $\pm$ 0.358} & \textcolor{black}{-0.170 $\pm$ 0.391}\tabularnewline
\textcolor{black}{DeepLTL} & \textcolor{black}{-0.075 $\pm$ 0.392} & \textcolor{black}{-0.089 $\pm$ 0.298} & \textcolor{black}{-0.243 $\pm$ 0.432}\tabularnewline
\hline 
\textcolor{black}{RGTN} & \textcolor{black}{0.991 $\pm$ 0.124} & \textcolor{black}{0.905 $\pm$ 0.172} & \textcolor{black}{0.614 $\pm$ 0.377}\tabularnewline
\textcolor{black}{LOTUS} & \textbf{\textcolor{black}{0.993}}\textcolor{black}{{} $\pm$ }\textbf{\textcolor{black}{0.113}} & \textbf{\textcolor{black}{0.915 $\pm$ 0.126}} & \textbf{\textcolor{black}{0.648}}\textcolor{black}{{} $\pm$ }\textbf{\textcolor{black}{0.448}}\tabularnewline
\hline 
\textcolor{black}{Task} & \multicolumn{3}{c}{\textbf{\textcolor{black}{LetterEnv}}}\tabularnewline
\hline 
\textcolor{black}{Type} & \multicolumn{3}{c}{\textcolor{black}{Partially-Ordered Tasks}}\tabularnewline
\hline 
\textcolor{black}{Split} & \textcolor{black}{I.I.D} & \textcolor{black}{Depth ($\uparrow$)} & \textcolor{black}{Conjuncts ($\uparrow$)}\tabularnewline
\hline 
\textcolor{black}{No LTL} & \textcolor{black}{0.988 $\pm$ 0.063} & \textcolor{black}{0.048 $\pm$ 0.234} & \textcolor{black}{0.265 $\pm$ 0.365}\tabularnewline
\textcolor{black}{LSTM} & \textcolor{black}{0.995 $\pm$ 0.058} & \textcolor{black}{0.299 $\pm$ 0.461} & \textcolor{black}{0.557 $\pm$ 0.401}\tabularnewline
\textcolor{black}{T1TL} & \textcolor{black}{0.782 $\pm$ 0.296} & \textcolor{black}{0.046 $\pm$ 0.221} & \textcolor{black}{0.429 $\pm$ 0.403}\tabularnewline
\textcolor{black}{LTL2Action} & \textcolor{black}{1.000 $\pm$ 0.000} & \textcolor{black}{0.608 $\pm$ 0.386} & \textcolor{black}{0.729 $\pm$ 0.331}\tabularnewline
\textcolor{black}{HGT} & \textcolor{black}{1.000 $\pm$ 0.000} & \textcolor{black}{0.520 $\pm$ 0.396} & \textcolor{black}{0.433 $\pm$ 0.458}\tabularnewline
\textcolor{black}{cDFAs} & \textcolor{black}{0.986 $\pm$ 0.077} & \textcolor{black}{0.309 $\pm$ 0.461} & \textcolor{black}{0.298 $\pm$ 0.436}\tabularnewline
\textcolor{black}{DeepLTL} & \textcolor{black}{1.000 $\pm$ 0.000} & \textcolor{black}{0.705 $\pm$ 0.372} & \textcolor{black}{0.893 $\pm$ 0.195}\tabularnewline
\hline 
\textcolor{black}{RGTN} & \textcolor{black}{1.000 $\pm$ 0.000} & \textcolor{black}{0.731 $\pm$ 0.350} & \textcolor{black}{0.837 $\pm$ 0.214}\tabularnewline
\textcolor{black}{LOTUS} & \textbf{\textcolor{black}{1.000}}\textcolor{black}{{} $\pm$ }\textbf{\textcolor{black}{0.000}} & \textbf{\textcolor{black}{0.760}}\textcolor{black}{{} $\pm$ }\textbf{\textcolor{black}{0.332}} & \textbf{\textcolor{black}{0.934}}\textcolor{black}{{} $\pm$ }\textbf{\textcolor{black}{0.102}}\tabularnewline
\hline 
\textcolor{black}{Task} & \multicolumn{3}{c}{\textbf{\textcolor{black}{ZoneEnv}}}\tabularnewline
\hline 
\textcolor{black}{Type} & \multicolumn{3}{c}{\textcolor{black}{Avoidance Tasks}}\tabularnewline
\hline 
\textcolor{black}{Split} & \textcolor{black}{I.I.D} & \textcolor{black}{Depth ($\uparrow$)} & \textcolor{black}{Conjuncts ($\uparrow$)}\tabularnewline
\hline 
\textcolor{black}{No LTL} & \textcolor{black}{-0.229 $\pm$ 0.675} & \textcolor{black}{-0.415 $\pm$ 0.593} & \textcolor{black}{-0.539 $\pm$ 0.540}\tabularnewline
\textcolor{black}{LSTM} & \textcolor{black}{0.378 $\pm$ 0.605} & \textcolor{black}{0.010 $\pm$ 0.322} & \textcolor{black}{0.429 $\pm$ 0.610}\tabularnewline
\textcolor{black}{T1TL} & \textcolor{black}{-0.004 $\pm$ 0.450} & \textcolor{black}{-0.025 $\pm$ 0.191} & \textcolor{black}{-0.092 $\pm$ 0.307}\tabularnewline
\textcolor{black}{LTL2Action} & \textcolor{black}{0.910 $\pm$ 0.170} & \textcolor{black}{0.928 $\pm$ 0.194} & \textcolor{black}{0.547 $\pm$ 0.535}\tabularnewline
\textcolor{black}{HGT} & \textcolor{black}{0.992 $\pm$ 0.065} & \textcolor{black}{0.960 + 0.254} & \textcolor{black}{0.576 $\pm$ 0.609}\tabularnewline
\textcolor{black}{cDFAs} & \textcolor{black}{0.811 $\pm$ 0.235} & \textcolor{black}{0.629 $\pm$ 0.416} & \textcolor{black}{0.605 $\pm$ 0.652}\tabularnewline
\textcolor{black}{DeepLTL} & \textcolor{black}{0.898 $\pm$ 0.195} & \textcolor{black}{0.908 $\pm$ 0.204} & \textcolor{black}{0.539 $\pm$ 0.557}\tabularnewline
\hline 
\textcolor{black}{RGTN} & \textcolor{black}{0.975 $\pm$ 0.136} & \textcolor{black}{0.968 $\pm$ 0.149} & \textcolor{black}{0.719 $\pm$ 0.497}\tabularnewline
\textcolor{black}{LOTUS} & \textbf{\textcolor{black}{1.000}}\textcolor{black}{{} $\pm$ }\textbf{\textcolor{black}{0.000}} & \textbf{\textcolor{black}{1.000 $\pm$ 0.000}} & \textbf{\textcolor{black}{0.743 $\pm$ 0.460}}\tabularnewline
\hline 
\textcolor{black}{Task} & \multicolumn{3}{c}{\textbf{\textcolor{black}{ZoneEnv}}}\tabularnewline
\hline 
\textcolor{black}{Type} & \multicolumn{3}{c}{\textcolor{black}{Partially-Ordered Tasks}}\tabularnewline
\hline 
\textcolor{black}{Split} & \textcolor{black}{I.I.D} & \textcolor{black}{Depth ($\uparrow$)} & \textcolor{black}{Conjuncts ($\uparrow$)}\tabularnewline
\hline 
\textcolor{black}{No LTL} & \textcolor{black}{0.738 $\pm$ 0.434} & \textcolor{black}{0.089 $\pm$ 0.270} & \textcolor{black}{0.219 $\pm$ 0.419}\tabularnewline
\textcolor{black}{LSTM} & \textcolor{black}{0.786 $\pm$ 0.385} & \textcolor{black}{0.094 $\pm$ 0.285} & \textcolor{black}{0.234 $\pm$ 0.425}\tabularnewline
\textcolor{black}{T1TL} & \textcolor{black}{0.644 $\pm$ 0.475} & \textcolor{black}{0.031 $\pm$ 0.191} & \textcolor{black}{0.075 $\pm$ 0.280}\tabularnewline
\textcolor{black}{LTL2Action} & \textcolor{black}{0.991 $\pm$ 0.063} & \textcolor{black}{0.608 $\pm$ 0.487} & \textcolor{black}{0.739 $\pm$ 0.429}\tabularnewline
\textcolor{black}{HGT} & \textcolor{black}{1.000 $\pm$ 0.000} & \textcolor{black}{0.614 $\pm$ 0.483} & \textcolor{black}{0.822 $\pm$ 0.385}\tabularnewline
\textcolor{black}{cDFAs} & \textcolor{black}{0.995 $\pm$ 0.050} & \textcolor{black}{0.527 $\pm$ 0.496} & \textcolor{black}{0.864 $\pm$ 0.350}\tabularnewline
\textcolor{black}{DeepLTL} & \textcolor{black}{1.000 $\pm$ 0.000} & \textcolor{black}{0.681 $\pm$ 0.432} & \textcolor{black}{0.842 $\pm$ 0.385}\tabularnewline
\hline 
\textcolor{black}{RGTN} & \textcolor{black}{1.000 $\pm$ 0.000} & \textcolor{black}{0.683 $\pm$ 0.425} & \textcolor{black}{0.839 $\pm$ 0.392}\tabularnewline
\textcolor{black}{LOTUS} & \textbf{\textcolor{black}{1.000 $\pm$ 0.000}} & \textbf{\textcolor{black}{0.709 $\pm$ 0.405}} & \textbf{\textcolor{black}{0.871 $\pm$ 0.312}}\tabularnewline
\hline 
\end{tabular}}
\end{table}

\textit{1) Motion Planning Generalization.} We further evaluate
the performance of different algorithms on LetterEnv and ZoneEnv under the original task distribution (i.i.d), as well as under more challenging settings with longer sequences (Depth$\uparrow$)
and increased conjunctions (Conjuncts$\uparrow$) across two task types:
Partially-Ordered and Avoidance Tasks. Detailed task settings can
be found on our \href{https://lotus-website.github.io/}{website}. 

We \textcolor{black}{present the corresponding generalization results
in Table \ref{tab:Genal2ZoneEnv}. From Table \ref{tab:Genal2ZoneEnv},
several observations can be made. (1) Within the
initial task distribution, LOTUS consistently achieves higher returns
and lower variance across all task scenarios, indicating that it not
only performs better but also exhibits greater stability. And our
RGTN algorithm further demonstrates the effectiveness of the architecture.
(2) With the increase of the depth and
the conjunctions of tasks, algorithms
with directed graph modeling (R-GCN, cDFA, HGT, RGTN and LOTUS)
exhibit more robust performance, highlighting the effectiveness of
graph neural network for motion planning in multi-task scenarios.
However, compared with other algorithms, cDFA shows a more pronounced
decline in returns relative to the initial tasks across both task
settings. (3) We observe that Transformer from \cite{zhang2023exploiting}
fails to achieve satisfactory performance in any multi-task setting.
This may be attributed to its architectural limitations in accommodating
heterogeneous task representations, leading to representational conflicts
and imprecise feature modeling. (4) Both LTL2Action and HGT perform
competitive performance across various task scenarios. Nevertheless,
the former achieves superior results in discrete environments, whereas
the latter performs better in continuous environments. This divergence
may stem from the following differences: LTL2Action depends more on
explicit relational constraints and the stability of local neighbor
aggregation, and HGT benefits from modeling similarity among dynamic
features and capturing global dependencies. (5) DeepLTL demonstrates
strong performance across most scenarios, highlighting the importance
of sequence modules and permutation invariance for task representation
encoding. (6) In most multi-task settings, LOTUS achieves the best
returns along with greater stability. This advantage primarily arises
from the strong representational capacity of the RGTN architecture
and the stabilizing effect introduced by the bisimulation metric during
representation optimization.}

\begin{table*}
\caption{\label{tab:Genal2DexArt}Comparison of success rate on four
manipulation scenarios \textcolor{black}{ (Toilet, Trashcan, Faucet,
and Dispenser tasks)}
for both seen and unseen objects.}

\centering{}\resizebox{0.89\textwidth}{!}{%%
\begin{tabular}{c|cc|cc|cccc}
\hline 
Task & \multicolumn{2}{c|}{Toilet} & \multicolumn{2}{c|}{Trashcan} & \multicolumn{2}{c}{Faucet} & \multicolumn{2}{c}{Dispenser}\tabularnewline
\hline 
Split & Seen & Unseen & Seen & Unseen & Seen & Unseen & Seen & Unseen\tabularnewline
\hline 
DFA & 0.815 $\pm$ 0.062 & 0.584 $\pm$ 0.069 & 0.643 $\pm$ 0.065 & 0.494 $\pm$ 0.008 & 0.687 $\pm$ 0.101 & 0.509 $\pm$ 0.112 & 0.746 $\pm$ 0.057 & 0.220 $\pm$ 0.031\tabularnewline
LSTM & 0.811 $\pm$ 0.033 & 0.566 $\pm$ 0.056 & 0.733 $\pm$ 0.061 & 0.566 $\pm$ 0.147 & 0.799 $\pm$ 0.066 & 0.609 $\pm$ 0.097 & 0.656 $\pm$ 0.052 & 0.258 $\pm$ 0.025\tabularnewline
T2TL & 0.844 $\pm$ 0.028 & 0.695 $\pm$ 0.116 & 0.570 $\pm$ 0.135 & 0.444 $\pm$ 0.034 & 0.842 $\pm$ 0.062 & 0.657 $\pm$ 0.104 & 0.733 $\pm$ 0.037 & 0.249 $\pm$ 0.071\tabularnewline
LTL2Action & 0.786 $\pm$ 0.074 & 0.603 $\pm$ 0.067 & 0.513 $\pm$ 0.120 & 0.477 $\pm$ 0.156 & 0.763 $\pm$ 0.064 & 0.595 $\pm$ 0.007 & 0.647 $\pm$ 0.105 & 0.171 $\pm$ 0.052\tabularnewline
HGT & 0.831 $\pm$ 0.030 & 0.651 $\pm$ 0.038 & 0.646 $\pm$ 0.057 & 0.555 $\pm$ 0.150 & 0.775 $\pm$ 0.075 & 0.585 $\pm$ 0.102 & 0.746 $\pm$ 0.024 & 0.271 $\pm$ 0.033\tabularnewline
\textcolor{black}{cDFA} & \textcolor{black}{0.840 $\pm$ 0.060} & \textcolor{black}{0.644 $\pm$ 0.051} & \textcolor{black}{0.509 $\pm$ 0.064} & \textcolor{black}{0.405 $\pm$ 0.122} & \textcolor{black}{0.866 $\pm$ 0.018} & \textcolor{black}{0.709 $\pm$ 0.076} & \textcolor{black}{0.696 $\pm$ 0.055} & \textcolor{black}{0.233 $\pm$ 0.031}\tabularnewline
\textcolor{black}{DeepLTL} & \textcolor{black}{0.789 $\pm$ 0.064} & \textcolor{black}{0.609 $\pm$ 0.012} & \textcolor{black}{0.736 $\pm$ 0.020} & \textcolor{black}{0.627 $\pm$ 0.047} & \textcolor{black}{0.590 $\pm$ 0.114} & \textcolor{black}{0.547 $\pm$ 0.007} & \textcolor{black}{0.710 $\pm$ 0.100} & \textcolor{black}{0.254 $\pm$ 0.071}\tabularnewline
\hline 
RGTN & 0.880 $\pm$ 0.020 & 0.590 $\pm$ 0.007 & 0.477 $\pm$ 0.181 & 0.505 $\pm$ 0.020 & 0.880 $\pm$ 0.012 & 0.747 $\pm$ 0.029 & 0.736 $\pm$ 0.032 & 0.241 $\pm$ 0.032\tabularnewline
LOTUS & \textbf{0.888 $\pm$ 0.004} & \textbf{0.728 $\pm$ 0.093} & \textbf{0.756 $\pm$ 0.023} & \textbf{0.644 $\pm$ 0.008} & \textbf{0.920 $\pm$ 0.022} & \textbf{0.794 $\pm$ 0.013} & \textbf{0.806 $\pm$ 0.041} & \textbf{0.337 $\pm$ 0.035}\tabularnewline
\hline 
\end{tabular}}
\end{table*}

\textit{2) Manipulation Generalization.} To evaluate the generalization performance of different task representation algorithms in guiding agents at the manipulation level, experiments are conducted in four articulated environments based on \cite{bao2023dexart}.
From Table \ref{tab:Genal2DexArt}, we have the following observations. (1) Compared with Table \ref{tab:Genal2ZoneEnv}, Transformer demonstrates stronger
generalization in manipulation environments, likely because tasks in these settings involve fewer sub-goals and display clearer structural regularities, making it easier for a lightweight encoder to extract essential features.
(2) HGT achieves better generalization than
LTL2Action, underscoring the effectiveness of the MHA method for the generalization of manipulation tasks. (3) RGTN consistently outperforms other baselines across most task scenarios, demonstrating the effectiveness of integrating global node and local edge-type features. (4) \textcolor{black}{LOTUS achieves the highest
returns and improved stability across most multi-task settings. This
performance gain stems from two key factors: the strong representational
capacity of the RGTN architecture and the stabilizing effect of the
bisimulation metric during representation optimization.}

\subsection{Quantitative Evaluation}

\begin{table}
\caption{\label{tab:Genal2ZoneEnv_R1C3}\textcolor{black}{P-values obtained
for other algorithms relative to LOTUS via the Wilcoxon signed-rank
test in }LetterWorld and ZoneEnv tasks. (/ indicates the same value
as LOTUS, so no Wilcoxon test is performed.)}

\centering{}\resizebox{0.40\textwidth}{!}{%%
\begin{tabular}{c|c|c|c|c|c|c}
\hline 
\textcolor{black}{Task} & \multicolumn{6}{c}{\textbf{\textcolor{black}{LetterWorld}}}\tabularnewline
\hline 
\textcolor{black}{Type} & \multicolumn{3}{c|}{\textcolor{black}{Avoidance Tasks}} & \multicolumn{3}{c}{\textcolor{black}{Partially-Ordered Tasks}}\tabularnewline
\hline 
\textcolor{black}{p-value ($\downarrow$)} & \textcolor{black}{I.I.D} & \textcolor{black}{Depth ($\uparrow$)} & \textcolor{black}{{} ($\uparrow$)} & \textcolor{black}{I.I.D} & \textcolor{black}{Depth ($\uparrow$)} & \textcolor{black}{{} ($\uparrow$)}\tabularnewline
\hline 
\textcolor{black}{No LTL} & \textbf{\textcolor{black}{<0.001}} & \textbf{\textcolor{black}{<0.001}} & \textbf{\textcolor{black}{<0.001}} & \textbf{\textcolor{black}{0.001}} & \textbf{\textcolor{black}{<0.001}} & \textbf{\textcolor{black}{<0.001}}\tabularnewline
\textcolor{black}{LSTM} & \textcolor{black}{0.164} & \textbf{\textcolor{black}{0.001}} & \textcolor{black}{0.012} & \textcolor{black}{0.008} & \textbf{\textcolor{black}{0.004}} & \textbf{\textcolor{black}{<0.001}}\tabularnewline
\textcolor{black}{T1TL} & \textbf{\textcolor{black}{<0.001}} & \textbf{\textcolor{black}{<0.001}} & \textbf{\textcolor{black}{<0.001}} & \textcolor{black}{0.005} & \textbf{\textcolor{black}{<0.001}} & \textbf{\textcolor{black}{<0.001}}\tabularnewline
\textcolor{black}{LTL2Action} & \textcolor{black}{0.869} & \textcolor{black}{0.956} & \textcolor{black}{0.784} & \textcolor{black}{/} & \textcolor{black}{0.245} & \textbf{\textcolor{black}{0.001}}\tabularnewline
\textcolor{black}{HGT} & \textcolor{black}{0.869} & \textcolor{black}{0.277} & \textcolor{black}{0.595} & \textcolor{black}{/} & \textcolor{black}{0.164} & \textbf{\textcolor{black}{<0.001}}\tabularnewline
\textcolor{black}{cDFAs} & \textbf{\textcolor{black}{<0.001}} & \textbf{\textcolor{black}{<0.001}} & \textbf{\textcolor{black}{<0.001}} & \textcolor{black}{0.595} & \textbf{\textcolor{black}{0.004}} & \textbf{\textcolor{black}{<0.001}}\tabularnewline
\textcolor{black}{DeepLTL} & \textbf{\textcolor{black}{<0.001}} & \textbf{\textcolor{black}{<0.001}} & \textbf{\textcolor{black}{<0.001}} & \textcolor{black}{/} & \textcolor{black}{0.812} & \textcolor{black}{0.985}\tabularnewline
\textcolor{black}{RGTN} & \textcolor{black}{0.153} & \textcolor{black}{0.176} & \textcolor{black}{0.261} & \textcolor{black}{/} & \textcolor{black}{0.189} & \textcolor{black}{0.898}\tabularnewline
\hline 
\textcolor{black}{Task} & \multicolumn{6}{c}{\textbf{\textcolor{black}{ZoneEnv}}}\tabularnewline
\hline 
\textcolor{black}{Type} & \multicolumn{3}{c|}{\textcolor{black}{Avoidance Tasks}} & \multicolumn{3}{c}{\textcolor{black}{Partially-Ordered Tasks}}\tabularnewline
\hline 
\textcolor{black}{Split} & \textcolor{black}{I.I.D} & \textcolor{black}{Depth ($\uparrow$)} & \textcolor{black}{{} ($\uparrow$)} & \textcolor{black}{I.I.D} & \textcolor{black}{Depth ($\uparrow$)} & \textcolor{black}{{} ($\uparrow$)}\tabularnewline
\hline 
\textcolor{black}{No LTL} & \textbf{\textcolor{black}{<0.001}} & \textbf{\textcolor{black}{<0.001}} & \textbf{\textcolor{black}{<0.001}} & \textbf{\textcolor{black}{0.001}} & \textbf{\textcolor{black}{<0.001}} & \textbf{\textcolor{black}{<0.001}}\tabularnewline
\textcolor{black}{LSTM} & \textbf{\textcolor{black}{<0.001}} & \textbf{\textcolor{black}{<0.001}} & \textcolor{black}{0.176} & \textcolor{black}{0.008} & \textbf{\textcolor{black}{<0.001}} & \textbf{\textcolor{black}{<0.001}}\tabularnewline
\textcolor{black}{T1TL} & \textbf{\textcolor{black}{<0.001}} & \textbf{\textcolor{black}{<0.001}} & \textbf{\textcolor{black}{<0.001}} & \textcolor{black}{0.005} & \textbf{\textcolor{black}{<0.001}} & \textbf{\textcolor{black}{<0.001}}\tabularnewline
\textcolor{black}{LTL2Action} & \textbf{\textcolor{black}{0.002}} & \textcolor{black}{0.015} & \textcolor{black}{0.294} & \textcolor{black}{0.277} & \textcolor{black}{0.812} & \textcolor{black}{0.368}\tabularnewline
\textcolor{black}{HGT} & \textcolor{black}{0.452} & \textcolor{black}{0.409} & \textcolor{black}{0.595} & \textcolor{black}{/} & \textcolor{black}{0.812} & \textcolor{black}{0.985}\tabularnewline
\textcolor{black}{cDFAs} & \textbf{\textcolor{black}{0.001}} & \textbf{\textcolor{black}{<0.001}} & \textcolor{black}{0.545} & \textcolor{black}{0.595} & \textcolor{black}{0.498} & \textcolor{black}{0.985}\tabularnewline
\textcolor{black}{DeepLTL} & \textcolor{black}{0.012} & \textcolor{black}{0.048} & \textcolor{black}{0.728} & \textcolor{black}{/} & \textcolor{black}{0.647} & \textcolor{black}{0.647}\tabularnewline
\textcolor{black}{RGTN} & \textcolor{black}{0.728} & \textcolor{black}{0.784} & \textcolor{black}{0.164} & \textcolor{black}{/} & \textcolor{black}{0.176} & \textcolor{black}{0.202}\tabularnewline
\hline 
\end{tabular}}
\end{table}
\begin{table}
\caption{\label{tab:Genal2DexArt_R1C3}\textcolor{black}{P-values obtained
for other algorithms relative to LOTUS via the Wilcoxon signed-rank
test in} four manipulation scenarios (Toilet, Trashcan, Faucet, and
Dispenser tasks) for Both Seen and Unseen Objects.}

\centering{}\resizebox{0.45\textwidth}{!}{%%
\begin{tabular}{c|cc|cc|cc|cc}
\hline 
\textcolor{black}{Task} & \multicolumn{2}{c|}{\textcolor{black}{Toilet}} & \multicolumn{2}{c|}{\textcolor{black}{Trashcan}} & \multicolumn{2}{c|}{\textcolor{black}{Faucet}} & \multicolumn{2}{c}{\textcolor{black}{Dispenser}}\tabularnewline
\hline 
\textcolor{black}{Split} & \textcolor{black}{Seen} & \textcolor{black}{Unseen} & \textcolor{black}{Seen} & \textcolor{black}{Unseen} & \textcolor{black}{Seen} & \textcolor{black}{Unseen} & \textcolor{black}{Seen} & \textcolor{black}{Unseen}\tabularnewline
\hline 
\textcolor{black}{DFA} & \textbf{\textcolor{black}{0.002}} & \textbf{\textcolor{black}{0.002}} & \textbf{\textcolor{black}{<0.001}} & \textbf{\textcolor{black}{<0.001}} & \textbf{\textcolor{black}{<0.001}} & \textbf{\textcolor{black}{<0.001}} & \textbf{\textcolor{black}{<0.001}} & \textbf{\textcolor{black}{<0.001}}\tabularnewline
\textcolor{black}{LSTM} & \textbf{\textcolor{black}{0.002}} & \textbf{\textcolor{black}{0.002}} & \textcolor{black}{0.202} & \textcolor{black}{0.017} & \textbf{\textcolor{black}{<0.001}} & \textbf{\textcolor{black}{<0.001}} & \textbf{\textcolor{black}{<0.001}} & \textbf{\textcolor{black}{<0.001}}\tabularnewline
\textcolor{black}{T2TL} & \textbf{\textcolor{black}{0.003}} & \textcolor{black}{0.084} & \textbf{\textcolor{black}{<0.001}} & \textbf{\textcolor{black}{<0.001}} & \textbf{\textcolor{black}{<0.001}} & \textbf{\textcolor{black}{<0.001}} & \textbf{\textcolor{black}{<0.001}} & \textbf{\textcolor{black}{<0.001}}\tabularnewline
\textcolor{black}{LTL2Action} & \textbf{\textcolor{black}{0.002}} & \textbf{\textcolor{black}{0.002}} & \textbf{\textcolor{black}{<0.001}} & \textbf{\textcolor{black}{<0.001}} & \textbf{\textcolor{black}{<0.001}} & \textbf{\textcolor{black}{<0.001}} & \textbf{\textcolor{black}{<0.001}} & \textbf{\textcolor{black}{<0.001}}\tabularnewline
\textcolor{black}{HGT} & \textbf{\textcolor{black}{0.002}} & \textbf{\textcolor{black}{0.002}} & \textbf{\textcolor{black}{<0.001}} & \textcolor{black}{0.021} & \textbf{\textcolor{black}{<0.001}} & \textbf{\textcolor{black}{<0.001}} & \textbf{\textcolor{black}{<0.001}} & \textbf{\textcolor{black}{<0.001}}\tabularnewline
\textcolor{black}{cDFA} & \textcolor{black}{0.037} & \textbf{\textcolor{black}{0.003}} & \textbf{\textcolor{black}{<0.001}} & \textbf{\textcolor{black}{<0.001}} & \textbf{\textcolor{black}{<0.001}} & \textbf{\textcolor{black}{<0.001}} & \textbf{\textcolor{black}{<0.001}} & \textbf{\textcolor{black}{<0.001}}\tabularnewline
\textcolor{black}{DeepLTL} & \textbf{\textcolor{black}{0.002}} & \textbf{\textcolor{black}{0.002}} & \textcolor{black}{0.089} & \textcolor{black}{0.409} & \textbf{\textcolor{black}{<0.001}} & \textbf{\textcolor{black}{<0.001}} & \textbf{\textcolor{black}{0.002}} & \textbf{\textcolor{black}{<0.001}}\tabularnewline
\textcolor{black}{RGTN} & \textcolor{black}{0.130} & \textbf{\textcolor{black}{0.002}} & \textbf{\textcolor{black}{<0.001}} & \textbf{\textcolor{black}{<0.001}} & \textbf{\textcolor{black}{<0.001}} & \textbf{\textcolor{black}{<0.001}} & \textbf{\textcolor{black}{<0.001}} & \textbf{\textcolor{black}{<0.001}}\tabularnewline
\hline 
\end{tabular}}
\end{table}
\textcolor{black}{To demonstrate the magnitude of LOTUS's relative
advantage over other algorithms, we further applied the Wilcoxon test
to quantitatively evaluate the results presented in Tables \ref{tab:Genal2ZoneEnv}
and \ref{tab:Genal2DexArt}. The results are shown in Tables \ref{tab:Genal2ZoneEnv_R1C3}
and \ref{tab:Genal2DexArt_R1C3}. / indicates the same value as LOTUS,
so no Wilcoxon test is performed.}

\textcolor{black}{Table \ref{tab:Genal2ZoneEnv_R1C3} reveals the following
observations: 1) In both LetterWorld and ZoneEnv task scenarios, using
p<0.05 as the threshold, LOTUS's advantages over other algorithms
are not significant in most cases. This is primarily because Depth
and Conjuncts inherently increase the challenge level of these task
scenarios. 2) RGTN exhibits higher p-values in most scenarios, indicating
that RGTN achieves higher returns across both LetterWorld and ZoneEnv
task settings.}

\textcolor{black}{From Table \ref{tab:Genal2DexArt_R1C3}, we can observe:
1) Using p<0.05 as the threshold, LOTUS significantly outperforms
other algorithms in most scenarios across the four operational tasks.
2) Although algorithms like T2TL, LSTM, and DeepLTL exhibit lower
overall success rates than LOTUS, they can achieve high success rates
in specific operational scenarios. For example, DeepLTL achieves a
p-value of 0.409 relative to LOTUS in the unseen set of the Trashcan
task.}

\subsection{Visualization of Trajectories in ZoneEnv}

\begin{figure}
\centering{} \includegraphics[scale=0.27]{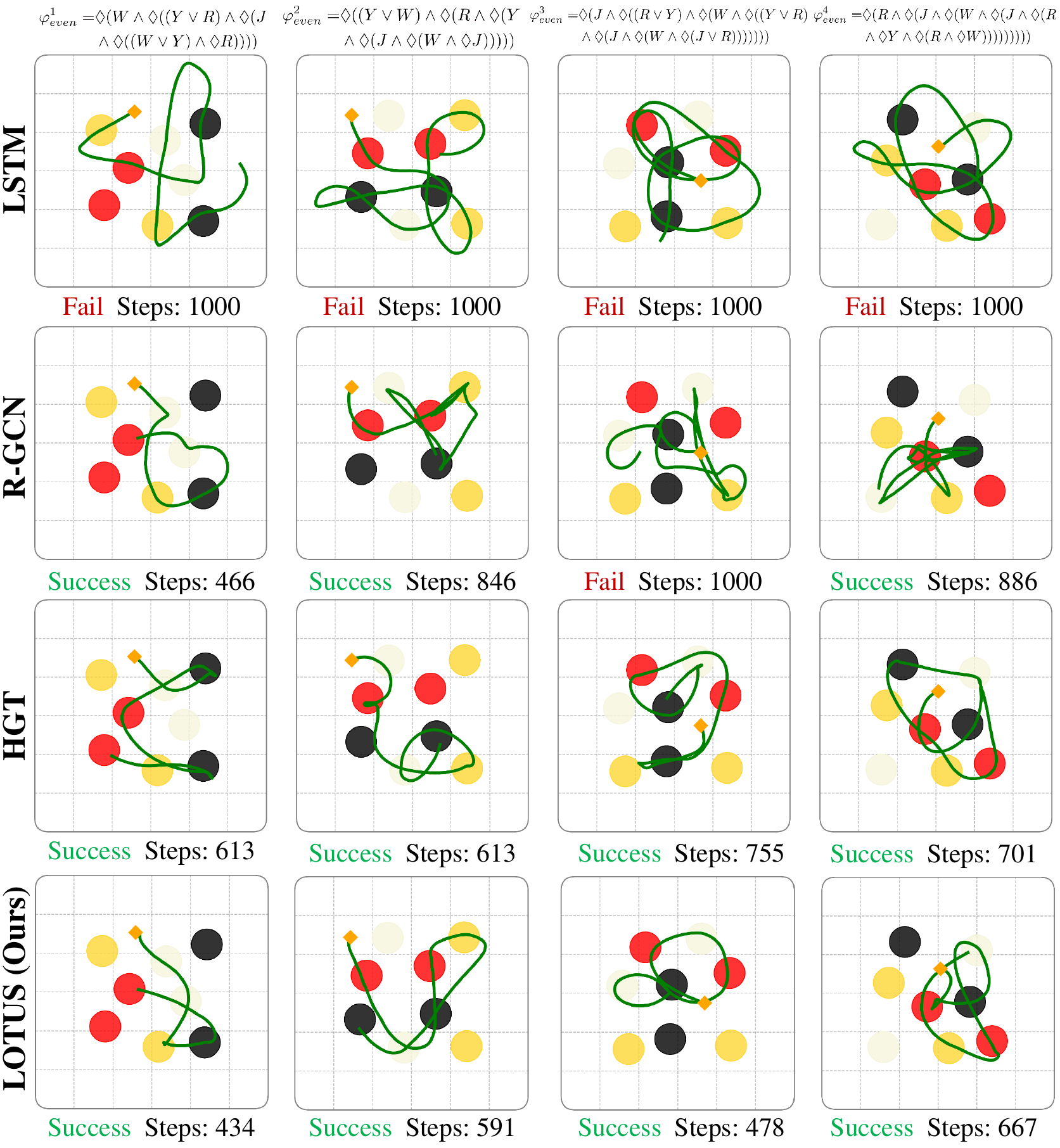}\caption{\label{fig:traj4zones} The trajectories
of different algorithms when tackling the increasingly partially-ordered tasks $\varphi_{even}^{1}$, $\varphi_{even}^{2}$, $\varphi_{even}^{3}$,
and $\varphi_{even}^{4}$ in ZoneEnv.}
\end{figure}
To further evaluate the generalization capabilities of different algorithms
in complex tasks, we design four increasingly complex  tasks $\varphi_{even}^{1}$,
$\varphi_{even}^{2}$, $\varphi_{even}^{3}$ and $\varphi_{even}^{4}$
from Partially-Ordered (Depth$\uparrow$) task, where $\varphi_{even}^{1}=\lozenge(W\wedge\lozenge((Y\vee R)\wedge\lozenge(J\wedge\lozenge((W\vee Y)\wedge\lozenge R))))$,
$\varphi_{even}^{2}=\lozenge((Y\vee W)\wedge\lozenge(R\wedge\lozenge(Y\wedge\lozenge(J\wedge\lozenge(W\wedge\lozenge J)))))$,
$\varphi_{even}^{3}=\lozenge(J\wedge\lozenge((R\vee Y)\wedge\lozenge(W\wedge\lozenge((Y\vee R)\wedge\lozenge(J\wedge\lozenge(W\wedge\lozenge(J\vee R)))))))$,
and $\varphi_{even}^{4}=\lozenge(R\wedge\lozenge(J\wedge\lozenge(W\wedge\lozenge(J\wedge\lozenge(R\wedge\lozenge Y\wedge\lozenge(R\wedge\lozenge W)))))))))$.
The physical interpretation of these instructions is provided on our
\href{https://lotus-website.github.io/}{website}. 
Fig. \ref{fig:traj4zones} presents the trajectory visualizations and the number of steps required by the different algorithms.

From Fig. \ref{fig:traj4zones}, we observe the following results.
(1) The agent guided by LSTM becomes ineffective in planning efficient solutions as the number of sub-goals continuous to  increase. 
(2) Although the agent guided by R-GCN can complete most tasks, it often exhausts all interaction steps when handling $\vee$ scenarios between different sub-goals, leading to task failures (e.g., $(Y\vee R)$ and $(J\vee R)$ in $\varphi_{even}^{3}$). 
(3) The agent guided by HGT can complete all tasks without failure, whereas our method (LOTUS) achieves successful task completion with fewer steps than all other methods across all task scenarios.

\subsection{Ablation Study}

\begin{table}
\caption{\label{tab:ablation2components}Success Rate of Different Methods \textcolor{black}{in Toilet and Faucet tasks}
when ablating different component designs.}

\centering{}\resizebox{0.41\textwidth}{!}{%%
\begin{tabular}{c|cc|cc}
\hline 
Task & \multicolumn{2}{c|}{Toilet} & \multicolumn{2}{c}{Faucet}\tabularnewline
\hline 
Split & Seen & Unseen & Seen & Unseen\tabularnewline
\hline 
No LTL & 0.850\textbf{ $\pm$ }0.01 & 0.550\textbf{ $\pm$ }0.010 & 0.790 \textpm{} 0.020 & 0.580 \textpm{} 0.070\tabularnewline
RGTN & 0.880 $\pm$ 0.020 & 0.590 $\pm$ 0.007 & 0.880 $\pm$ 0.012 & 0.747 $\pm$ 0.029\tabularnewline
TLPG & 0.774\textbf{ $\pm$} 0.031 & 0.569 $\pm$ 0.049 & 0.909\textbf{ $\pm$ }0.007 & 0.749\textbf{ $\pm$ }0.026\tabularnewline
LOTUS & \textbf{0.888 $\pm$ 0.004} & \textbf{0.728 $\pm$ 0.093} & \textbf{0.920 $\pm$ 0.022} & \textbf{0.794 $\pm$ 0.013}\tabularnewline
\hline 
\end{tabular}}
\end{table}
\textit{1) Component Ablation.} 
We design the following ablation study to evaluate
the influence of different components in No LTL, RGTN, TLPG, and LOTUS across diverse task scenarios. The results are presented in Table \ref{tab:ablation2components}. 
Several observations can be made. (1) Without task semantics provided by the LTL encoder, No LTL performs worst across all scenarios. 
(2) RGTN achieves competitive performance in all scenarios. By modeling the LTL encoder as a policy to sampling the corresponding representation, TLPG attains a higher success rate in the Faucet task than RGTN. 
(3) By using the bismulation metric to improve the stability of representations during training, LOTUS shows better generalization performance than both RGTN and TLPG.

\begin{table}
\caption{\label{tab:abaltion2arch}Success Rate of Different Methods \textcolor{black}{in Toilet and Faucet tasks} when ablating
different architecture designs.}

\centering{}\resizebox{0.41\textwidth}{!}{%%
\begin{tabular}{c|cc|cc}
\hline 
Task & \multicolumn{2}{c|}{Toilet} & \multicolumn{2}{c}{Faucet}\tabularnewline
\hline 
Split & Seen & Unseen & Seen & Unseen\tabularnewline
\hline 
$\mathrm{RGTN_{w/o}^{MRA}}$ & 0.754 $\pm$ 0.085 & 0.585 $\pm$ 0.144 & 0.754 $\pm$ 0.085 & 0.604 $\pm$ 0.144\tabularnewline
$\mathrm{RGTN_{w/o}^{MHA}}$ & 0.768 $\pm$ 0.072 & 0.581 $\pm$ 0.104 & 0.787 $\pm$ 0.059 & 0.594 $\pm$ 0.077\tabularnewline
$\mathrm{RGTN_{w/o}^{Residual}}$ & 0.819 $\pm$ 0.026 & 0.548 $\pm$ 0.011 & 0.856 $\pm$ 0.008 & 0.728 $\pm$ 0.030\tabularnewline
RGTN & \textbf{0.880 $\pm$ 0.020} & \textbf{0.590 $\pm$ 0.007} & \textbf{0.880 $\pm$ 0.012} & \textbf{0.747 $\pm$ 0.029}\tabularnewline
\hline 
\end{tabular}}
\end{table}
\textit{2) Architecture Ablation.} To further demonstrate the effectiveness of different architectural designs in improving the representation capabilities of RGTN, we conducted the following ablation experiments. 
(1) $\mathrm{RGTN_{w/o}^{MRA}}$, which removes MRA to assess the importance of Multi-Relation Attention. 
(2) $\mathrm{RGTN_{w/o}^{MHA}}$, which removes MHA to assess the importance of Multi-Head Attention.
(3) $\mathrm{RGTN_{w/o}^{Residual}}$, which removes the residual architecture from the Transformer to evaluate its contribution.

\begin{figure}
\centering{} \includegraphics[scale=0.135]{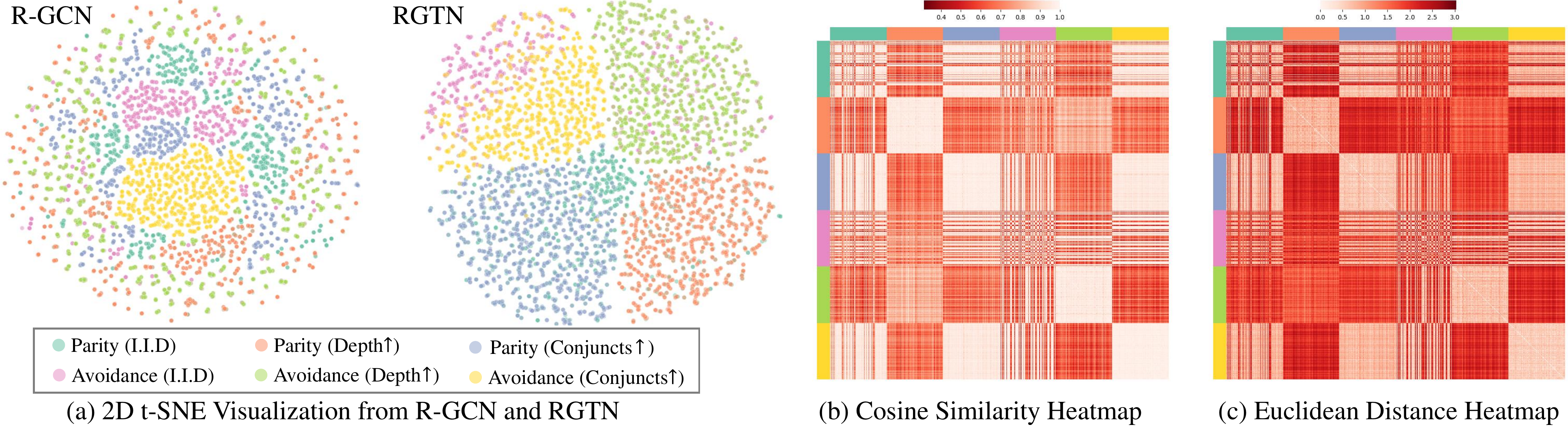}\caption{\textcolor{black}{\label{fig:visual_repre} Visualizations of (a)
the embedding spaces generated by R-GCN and RGTN, (b) }cosine similarity
heatmap \textcolor{black}{generated by RGTN, and (c) }Euclidean distance
heatmap \textcolor{black}{generated by RGTN.}}
\end{figure}

% \begin{figure}
% \centering{} \includegraphics[scale=0.125]{plot/S2R_Setting}\caption{\textcolor{black}{\label{fig:S2R_setting} The real-world experimental
% platform constructed for (a) $\varphi_{stack}$
% and (b) $\varphi_{cleanup}$.}}
% \end{figure}

\begin{figure}
\centering{} \includegraphics[scale=0.18]{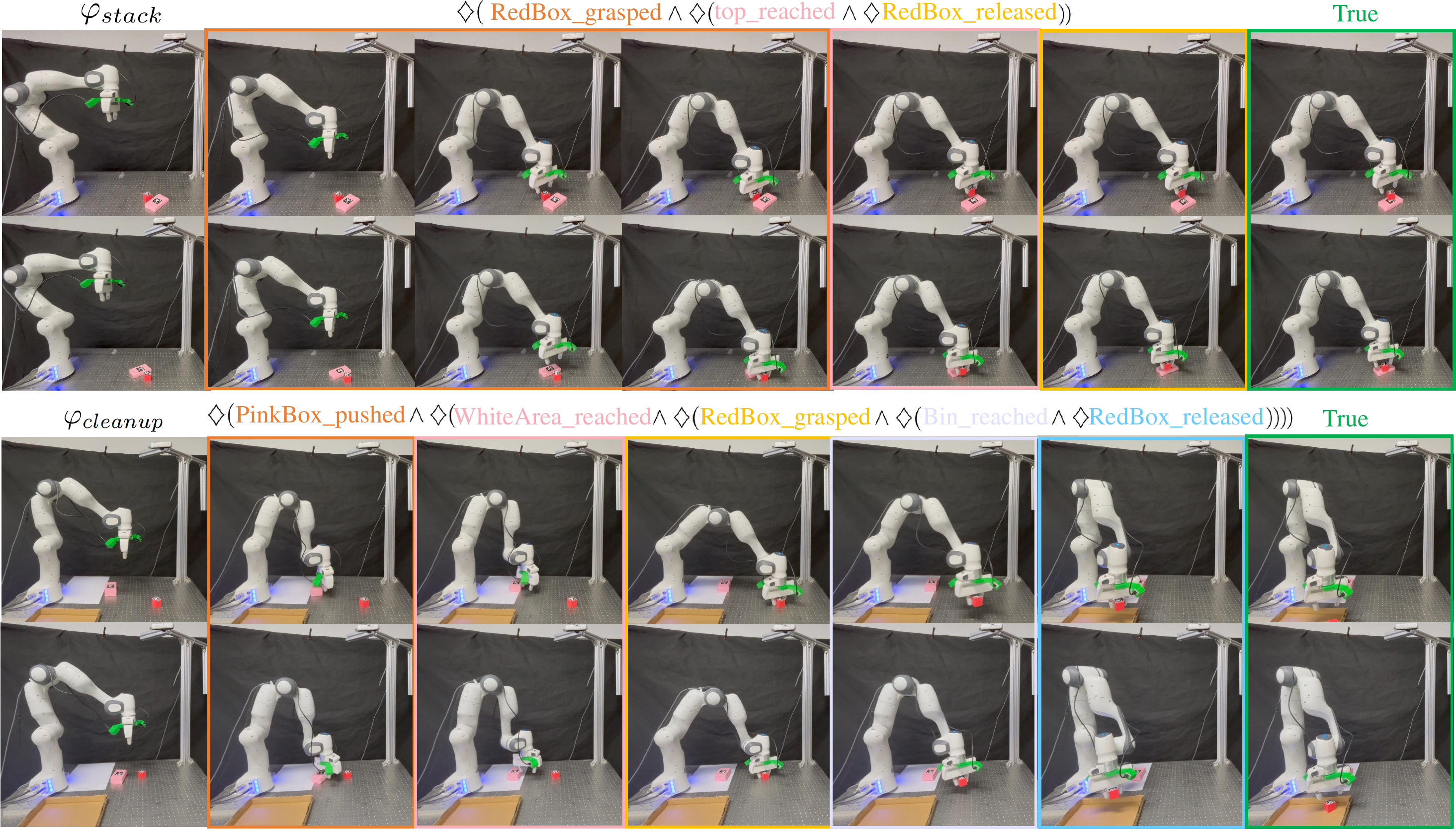}\caption{\textcolor{black}{\label{fig:S2R_snapshot} Snapshots of $\varphi_{stack}$
and $\varphi_{cleanup}$.}}
\end{figure}

% \begin{figure}
% \centering{} \includegraphics[scale=0.18]{plot/S2R_snapshot}\caption{\textcolor{black}{\label{fig:S2R_snapshot} Snapshots of $\varphi_{stack}$
% and $\varphi_{cleanup}$.}}
% \end{figure}

As shown in Table \ref{tab:abaltion2arch}, several observations can be made. 
(1) The performance of $\mathrm{RGTN_{w/o}^{MRA}}$ highlights the importance of Multi-Relation Attention in improving training efficiency and achieving strong performance on the Seen set across different tasks.
\textcolor{black}{This result demonstrates that MRA precisely
captures logical and temporal dependencies by modeling the four edge
types in LTL directed graphs, thereby learning this cross-scenario
invariance. This enhances the agent's learning ability for seen sets
and generalization capability for unseen sets. (2) Multi-Head Attention
(MHA) captures global information across multiple dimensions through
parallel attention heads. This approach not only integrates node information
to enhance the agent's understanding of overall task semantics but
also mitigates representation bias. Removing this component leads
to poor generalization performance, as evidenced by $\mathrm{RGTN_{w/o}^{MHA}}$
across all Unseen sets.} 
(3) Although removing the residual architecture has a smaller impact compared to MRA and MHA, the generalization success rate in the Unseen set of the Toilet task is 7.11\% lower than that of full RGTN model.

\begin{table}
\caption{\label{tab:abaltion2MRA}Success Rate of Different Methods in Toilet
and Faucet tasks when ablating different designs of MRA architecture}

\centering{}\resizebox{0.41\textwidth}{!}{%%
\begin{tabular}{c|cc|cc}
\hline 
\textcolor{black}{Task} & \multicolumn{2}{c|}{\textcolor{black}{Toilet}} & \multicolumn{2}{c}{\textcolor{black}{Faucet}}\tabularnewline
\hline 
\textcolor{black}{Split} & \textcolor{black}{Seen} & \textcolor{black}{Unseen} & \textcolor{black}{Seen} & \textcolor{black}{Unseen}\tabularnewline
\hline 
\textcolor{black}{$\mathrm{RGTN_{w/o}^{MRA}}$} & \textcolor{black}{0.754 $\pm$ 0.085} & \textcolor{black}{0.585 $\pm$ 0.144} & \textcolor{black}{0.754 $\pm$ 0.085} & \textcolor{black}{0.604 $\pm$ 0.144}\tabularnewline
\textcolor{black}{$\mathrm{RGTN_{fixed}^{MRA}}$} & \textcolor{black}{0.844 $\pm$ 0.048} & \textcolor{black}{0.644 $\pm$ 0.044} & \textcolor{black}{0.729 $\pm$ 0.103} & \textcolor{black}{0.537 $\pm$ 0.140}\tabularnewline
\textcolor{black}{$\mathrm{RGTN_{msg}^{MRA}}$} & \textcolor{black}{0.840 $\pm$ 0.054} & \textcolor{black}{0.590 $\pm$ 0.083} & \textcolor{black}{0.766 $\pm$ 0.035} & \textcolor{black}{0.609 $\pm$ 0.105}\tabularnewline
\textcolor{black}{RGTN} & \textbf{\textcolor{black}{0.880 $\pm$ 0.020}} & \textbf{\textcolor{black}{0.590 $\pm$ 0.007}} & \textbf{\textcolor{black}{0.880 $\pm$ 0.012}} & \textbf{\textcolor{black}{0.747 $\pm$ 0.029}}\tabularnewline
\hline 
\end{tabular}}
\end{table}
\textcolor{black}{\textit{3) Impact of Different Key Designs in MRA of RGTN.}
We compare the performance of different key designs in the Multi-Relation
Attention (MRA) on the Toilet and Faucet tasks. (1) $\mathrm{RGTN_{w/o}^{MRA}}$,
which removes MRA to evaluate the importance of MRA. (2) $\mathrm{RGTN_{msg}^{MRA}}$,
which shares information on edge types to reflect the importance of
independently transmitting edge type information in the original MRA.
(3) $\mathrm{RGTN_{fixed}^{MRA}}$, which fixes parameters on edge
types to reflect the importance of learnable edge-type parameters
for generalization in the original MRA. }

\textcolor{black}{As shown in Table \ref{tab:abaltion2MRA}, the following
observations can be made: 1) Other methods outperform $\mathrm{RGTN_{w/o}^{MRA}}$
on both seen and unseen sets, particularly in the Toilet task, demonstrating
MRA's importance for task representation. 2) Fixed edge type parameters
are not optimal for challenging tasks. The success rates of $\mathrm{RGTN_{fixed}^{MRA}}$
on both sets for the Faucet task are lower than other algorithms;
3) Although $\mathrm{RGTN_{msg}^{MRA}}$ achieves better performance
than $\mathrm{RGTN_{w/o}^{MRA}}$ and $\mathrm{RGTN_{fixed}^{MRA}}$
on some sets of the Toilet and Faucet tasks, its sharing of edge type
information weakens edge-type features. This further reduces its learning
efficiency and generalization performance on the Faucet task compared
to RGTN.}

\begin{table}
\caption{\label{tab:abaltion2MHA}Success Rate of Different Methods in Toilet
and Faucet tasks when ablating different designs of MHA architecture}

\centering{}\resizebox{0.41\textwidth}{!}{%%
\begin{tabular}{c|cc|cc}
\hline 
\textcolor{black}{Task} & \multicolumn{2}{c|}{\textcolor{black}{Toilet}} & \multicolumn{2}{c}{\textcolor{black}{Faucet}}\tabularnewline
\hline 
\textcolor{black}{Split} & \textcolor{black}{Seen} & \textcolor{black}{Unseen} & \textcolor{black}{Seen} & \textcolor{black}{Unseen}\tabularnewline
\hline 
\textcolor{black}{$\mathrm{RGTN_{w/o}^{MHA}}$} & \textcolor{black}{0.768 $\pm$ 0.072} & \textcolor{black}{0.581 $\pm$ 0.104} & \textcolor{black}{0.787 $\pm$ 0.059} & \textcolor{black}{0.594 $\pm$ 0.077}\tabularnewline
\textcolor{black}{$\mathrm{RGTN_{stand}^{MHA}}$} & \textcolor{black}{0.817 $\pm$ 0.049} & \textcolor{black}{0.605 $\pm$ 0.047} & \textcolor{black}{0.781 $\pm$ 0.032} & \textcolor{black}{0.652 $\pm$ 0.052}\tabularnewline
\textcolor{black}{$\mathrm{RGTN_{all}^{MHA}}$} & \textcolor{black}{0.846 $\pm$ 0.039} & \textcolor{black}{0.581 $\pm$ 0.066} & \textcolor{black}{0.788 $\pm$ 0.042} & \textcolor{black}{0.628 $\pm$ 0.020}\tabularnewline
\textcolor{black}{RGTN} & \textbf{\textcolor{black}{0.880 $\pm$ 0.020}} & \textbf{\textcolor{black}{0.590 $\pm$ 0.007}} & \textbf{\textcolor{black}{0.880 $\pm$ 0.012}} & \textbf{\textcolor{black}{0.747 $\pm$ 0.029}}\tabularnewline
\hline 
\end{tabular}}
\end{table}
\textcolor{black}{\textit{4) Impact of Different Key Designs in MHA of RGTN.}
We compare the performance of different key designs in the Multi-Head
Attention on the Toilet and Faucet tasks. (1) $\mathrm{RGTN_{w/o}^{MHA}}$,
which removes MHA to evaluate the importance of MHA. (2) $\mathrm{RGTN_{stand}^{MHA}}$,
which computes attention information as Transformer to reflect the
importance of independently propagating attention information on edge
types in the original MHA. (3) $\mathrm{RGTN_{all}^{MHA}}$, which
computes attention across all nodes per head to reflect the importance
of edge-based attention on adjacent nodes in the original MHA.}

\textcolor{black}{As shown in Table \ref{tab:abaltion2MHA}, the following
observations can be made: 1) $\mathrm{RGTN_{w/o}^{MHA}}$ performs
worse than other algorithms on both Toilet and Faucet tasks, demonstrating
MHA's importance for task representation. 2) $\mathrm{RGTN_{stand}^{MHA}}$
shows slightly better performance than $\mathrm{RGTN_{w/o}^{MHA}}$,
but its failure to compute edge-based feature attention on the head
limits its success rates on both tasks; 3) For $\mathrm{RGTN_{all}^{MHA}}$,
while computing attention across all nodes in the head achieves good
performance on the Toilet task, its success rate on challenging tasks
(e.g., Faucet task) declines by 10.45\% and 15.93\% relative to RGTN,
respectively. }

\begin{table}
\caption{\label{tab:abaltion2fuse}Success Rate of Different Methods in Toilet
and Faucet tasks when ablating different designs of the fusion information
module}

\centering{}\resizebox{0.45\textwidth}{!}{%%
\begin{tabular}{c|cc|cc}
\hline 
\textcolor{black}{Task} & \multicolumn{2}{c|}{\textcolor{black}{Toilet}} & \multicolumn{2}{c}{\textcolor{black}{Faucet}}\tabularnewline
\hline 
\textcolor{black}{Split} & \textcolor{black}{Seen} & \textcolor{black}{Unseen} & \textcolor{black}{Seen} & \textcolor{black}{Unseen}\tabularnewline
\hline 
\textcolor{black}{$\mathrm{RGTN_{add}^{fuse}}$} & \textcolor{black}{0.773}\textbf{\textcolor{black}{{} $\pm$ }}\textcolor{black}{0.040} & \textcolor{black}{0.557}\textbf{\textcolor{black}{{} $\pm$ }}\textcolor{black}{0.004} & \textcolor{black}{0.699}\textbf{\textcolor{black}{{} $\pm$ }}\textcolor{black}{0.038} & \textcolor{black}{0.475}\textbf{\textcolor{black}{{} $\pm$ }}\textcolor{black}{0.067}\tabularnewline
\textcolor{black}{$\mathrm{RGTN_{weighted}^{fuse}}$} & \textcolor{black}{0.782}\textbf{\textcolor{black}{{} $\pm$ }}\textcolor{black}{0.048} & \textcolor{black}{0.590}\textbf{\textcolor{black}{{} $\pm$ }}\textcolor{black}{0.077} & \textcolor{black}{0.781}\textbf{\textcolor{black}{{} $\pm$ }}\textcolor{black}{0.071} & \textcolor{black}{0.547}\textbf{\textcolor{black}{{} $\pm$ }}\textcolor{black}{0.075}\tabularnewline
\textcolor{black}{$\mathrm{RGTN_{concat}^{fuse}}$ (Ours)} & \textbf{\textcolor{black}{0.880 $\pm$ 0.020}} & \textbf{\textcolor{black}{0.590 $\pm$ 0.007}} & \textbf{\textcolor{black}{0.880 $\pm$ 0.012}} & \textbf{\textcolor{black}{0.747 $\pm$ 0.029}}\tabularnewline
\hline 
\end{tabular}}
\end{table}
\textcolor{black}{\textit{5) Impact of different attention fusion projection
module in RGTN.} We compare the performance of different fusion designs
in the Toilet and Faucet tasks. (1) $\mathrm{RGTN_{add}^{fuse}}$
employs a direct summation approach to fuse MRA and MHA information.
And (2) $\mathrm{RGTN_{weighted}^{fuse}}$ utilizes learnable weight
parameters to fuse MRA and MHA information.}

\textcolor{black}{As shown in Table \ref{tab:abaltion2fuse}, the following
observations can be observed: 1) When directly adding MRA and MHA
information, the performance of $\mathrm{RGTN_{add}^{fuse}}$ significantly
lags behind $\mathrm{RGTN_{weighted}^{fuse}}$ and $\mathrm{RGTN_{concat}^{fuse}}$.
This may stem from addition forcing the homogenization of distinct
features within LTL task semantics, diluting critical semantics and
hindering the precise capture of complex the logic in LTL formula.
2) When employing learnable weighting to fuse MRA and MHA information,
the performance of $\mathrm{RGTN_{weighted}^{fuse}}$ improves over
$\mathrm{RGTN_{add}^{fuse}}$ but remains inferior to $\mathrm{RGTN_{concat}^{fuse}}$.
This may stem from learnable weighting typically requiring more precise
manual parameter tuning for optimal results. Without such tuning,
it struggles to adapt to the complex semantics of LTL tasks, resulting
in insufficient precision and generalization in feature integration.
3) Our approach $\mathrm{RGTN_{concat}^{fuse}}$, compared to $\mathrm{RGTN_{add}^{fuse}}$,
fully preserves the semantic information of features across different
levels in LTL tasks. Unlike $\mathrm{RGTN_{weighted}^{fuse}}$, it
eliminates the need for redundant manual parameter tuning. By expanding
feature dimensionality, $\mathrm{RGTN_{concat}^{fuse}}$ enhances
expressive capacity while improving feature adaptability in generalized
scenarios.}

\subsection{Visualization of LTL Task Embedding Space}
To further visualize the representation capability of RGTN in the embedding space, we sampled six task types from LetterEnv: Avoidance (i.i.d), Avoidance (Depth$\uparrow$), Avoidance (Conjuncts$\uparrow$), Partially-Ordered (i.i.d), Partially-Ordered (Depth$\uparrow$), and Partially-Ordered (Conjuncts$\uparrow$).
Two methods were employed to assess task representation
performance: (1) applying t-SNE \cite{maaten2008visualizing} to project the embeddings into a two-dimensional space, and (2) computing the pairwise cosine similarity and Euclidean distance between embeddings.

The visualization results are shown in Fig. \ref{fig:visual_repre}. Several observation can be made. 
\textcolor{black}{(1) RGTN exhibits strong
task-type clustering ability, producing an embedding space that clearly
separates tasks with distinct structural properties. As shown in the
2D t-SNE visualization in Fig. \ref{fig:visual_repre}(a), Avoidance
Tasks and Partially-Ordered Tasks form two well-defined and non-overlapping
clusters, with embeddings of the same task type tightly grouped. In
contrast, R-GCN yields scattered and intermingled embeddings,
making task differentiation difficult. This result indicates that
RGTN effectively captures task structural semantics through MRA for
edge-type modeling and MHA for global dependency integration, thereby
enabling discriminative task representations that support downstream
generalization. (2) For partially ordered tasks, embeddings from the
i.i.d. and Depth (\textuparrow ) scenarios exhibit substantial overlap
in the t-SNE projection. Although the Depth (\textuparrow ) setting
increases task complexity by deepening sub-goal nesting, the underlying
hierarchical ordering of sub-goals remains unchanged. This overlap
indicates that RGTN encodes core structural semantics rather than
superficial propositional complexity, mapping tasks with shared logical
structure to similar representations. (3) The cosine similarity shown
in Fig. \ref{fig:visual_repre}(b) and Euclidean distance heatmaps
shown in Fig. \ref{fig:visual_repre}(c) quantitatively validate the
discriminability and generalization capacity of the embedding space.
Cosine similarities exceed 0.85 within the same task type (e.g., i.i.d.,
Depth (\textuparrow ), and Conjuncts (\textuparrow ) of Avoidance
tasks) and drop below 0.3 across different types, indicating clear
structural separation. Consistently, Euclidean distances are small
within task types (<1.2) but large across types (>2.8). Notably, out-of-distribution
tasks (e.g., Conjuncts (\textuparrow ) tasks) remain close to their
base task type (\ensuremath{\approx}1.5) while staying well separated
from others, enabling effective knowledge transfer based on structural
similarity and supporting strong generalization.}

\subsection{Computational Cost and Memory Usage }

\begin{table}
\caption{\label{tab: time_consume-R2}The computational cost (second) of different
methods consuming for each training epoch in LetterWorld and ZoneEnv
tasks}

\centering{}\resizebox{0.45\textwidth}{!}{%%
\begin{tabular}{c|c|c|c|c}
\hline 
\textcolor{black}{Task} & \multicolumn{2}{c|}{\textcolor{black}{LetterWorld (second)}} & \multicolumn{2}{c}{\textcolor{black}{ZoneEnv (second)}}\tabularnewline
\hline 
\textcolor{black}{Task Type} & \textcolor{black}{Avoidance} & \textcolor{black}{Partially-Ordered } & \textcolor{black}{Avoidance } & \textcolor{black}{Partially-Ordered}\tabularnewline
\hline 
\textcolor{black}{No LTL} & \textbf{\textcolor{black}{10.416 + 1.255}} & \textbf{\textcolor{black}{8.583 + 2.625}} & \textbf{\textcolor{black}{112.583 + 12.552}} & \textbf{\textcolor{black}{179.833 + 5.907}}\tabularnewline
\hline 
\textcolor{black}{LSTM} & \textcolor{black}{11.667 + 0.425} & \textcolor{black}{10.583 + 0.816} & \textcolor{black}{210.167 + 21.602} & \textcolor{black}{288.917 + 10.274}\tabularnewline
\textcolor{black}{T2TL} & \textcolor{black}{10.333 + 0.471} & \textcolor{black}{11.667 + 0.471} & \textcolor{black}{263.667 + 2.160} & \textcolor{black}{273.083 + 39.802}\tabularnewline
\textcolor{black}{LTL2Action} & \textcolor{black}{12.750 + 2.867} & \textcolor{black}{13.667 + 3.091} & \textcolor{black}{258.583 + 10.198} & \textcolor{black}{308.000 + 10.677}\tabularnewline
\textcolor{black}{HGT} & \textcolor{black}{15.435 + 3.658} & \textcolor{black}{14.333 + 4.784} & \textcolor{black}{292.917 + 21.297} & \textcolor{black}{243.417 + 8.340}\tabularnewline
\textcolor{black}{cDFAs} & \textcolor{black}{17.250 + 0.943} & \textcolor{black}{12.583 + 0.943} & \textcolor{black}{344.583 + 33.885} & \textcolor{black}{314.250 + 18.264}\tabularnewline
\textcolor{black}{DeepLTL} & \textcolor{black}{16.585 $\pm$ 1.276} & \textcolor{black}{15.525 $\pm$ 0.752} & \textcolor{black}{249.667 $\pm$ 10.277} & \textcolor{black}{256.667 + 13.257}\tabularnewline
\hline 
\textcolor{black}{RGTN} & \textcolor{black}{15.333 + 4.110} & \textcolor{black}{10.667 + 2.160} & \textcolor{black}{290.917 + 19.602} & \textcolor{black}{305.833 + 15.578}\tabularnewline
\textcolor{black}{TLPG} & \textcolor{black}{16.333 + 2.843} & \textcolor{black}{14.250 + 2.364} & \textcolor{black}{296.333 + 10.164} & \textcolor{black}{312.667 + 5.276}\tabularnewline
\textcolor{black}{LOTUS} & \textcolor{black}{19.000 + 1.633} & \textcolor{black}{17.167 + 4.710} & \textcolor{black}{300.583 + 12.499} & \textcolor{black}{317.583 + 7.483}\tabularnewline
\hline 
\end{tabular}}
\end{table}
\begin{table}
\caption{\label{tab: gpu_consume-R2}The GPU memory usage (GB) of different methods for each training epoch in LetterWorld and ZoneEnv
tasks}

\centering{}\resizebox{0.4\textwidth}{!}{%%
\begin{tabular}{c|c|c|c|c}
\hline 
\textcolor{black}{Task} & \multicolumn{2}{c|}{\textcolor{black}{LetterWorld}} & \multicolumn{2}{c}{\textcolor{black}{ZoneEnv}}\tabularnewline
\hline 
\textcolor{black}{Type} & \textcolor{black}{Avoidance} & \textcolor{black}{Partially-Ordered} & \textcolor{black}{Avoidance} & \textcolor{black}{Partially-Ordered }\tabularnewline
\hline 
\textcolor{black}{No LTL} & \textcolor{black}{0.770} & \textcolor{black}{0.770} & \textcolor{black}{0.620} & \textcolor{black}{0.620}\tabularnewline
\textcolor{black}{LSTM} & \textcolor{black}{0.660} & \textcolor{black}{0.602} & \textcolor{black}{0.645} & \textcolor{black}{0.529}\tabularnewline
\textcolor{black}{T2TL} & \textcolor{black}{0.732} & \textcolor{black}{0.992} & \textcolor{black}{0.732} & \textcolor{black}{0.918}\tabularnewline
\textcolor{black}{LTL2Action} & \textcolor{black}{0.488} & \textcolor{black}{0.490} & \textcolor{black}{0.990} & \textcolor{black}{0.949}\tabularnewline
\textcolor{black}{HGT} & \textcolor{black}{0.649} & \textcolor{black}{0.672} & \textcolor{black}{1.215} & \textcolor{black}{1.066}\tabularnewline
\textcolor{black}{cDFAs} & \textcolor{black}{1.084} & \textcolor{black}{1.093} & \textcolor{black}{1.213} & \textcolor{black}{1.232}\tabularnewline
\textcolor{black}{DeepLTL} & \textcolor{black}{0.787} & \textcolor{black}{0.787} & \textcolor{black}{0.808} & \textcolor{black}{0.828}\tabularnewline
\hline 
\textcolor{black}{RGTN} & \textcolor{black}{0.656} & \textcolor{black}{0.668} & \textcolor{black}{1.070} & \textcolor{black}{1.130}\tabularnewline
\textcolor{black}{TLPG} & \textcolor{black}{0.793} & \textcolor{black}{0.764} & \textcolor{black}{1.194} & \textcolor{black}{1.247}\tabularnewline
\textcolor{black}{LOTUS} & \textcolor{black}{0.955} & \textcolor{black}{0.920} & \textcolor{black}{1.297} & \textcolor{black}{1.369}\tabularnewline
\hline 
\end{tabular}}
\end{table}
\textcolor{black}{We analyze the time consumed per epoch and GPU utilization
during algorithm execution across two task types: LetterWorld and
ZoneEnv. Results are presented in Tables \ref{tab: time_consume-R2}and
\ref{tab: gpu_consume-R2}. Observations include: 1) Algorithms without LTL representation guidance exhibit the lowest time consumption
across all four task scenarios. While it
reduces time consumption, its performance is also degraded compared with other algorithms.
2) We note that LOTUS does not significantly increase runtime compared
to other algorithms, owing to our optimization of computation and
update processes based on the DGL \cite{wang2019dgl}. 3) RGTN achieves lower
GPU utilization across all four scenarios while maintaining comparable
performance to other algorithms. 4) Even with the additional computational
cost introduced by LOTUS's bisimulation metric, it maintains low GPU
utilization across the aforementioned scenarios and runs normally
on most graphics cards.}

\subsection{Longer Task Evaluation}

\begin{table}
\caption{\label{tab:long_ltl_tasks_with_bs}Total Return performance of Different
Methods in deepen Partially-Ordered Tasks of LetterWorld.}

\centering{}\resizebox{0.45\textwidth}{!}{%%
\begin{tabular}{c|cc|cc}
\hline 
\textcolor{black}{Partially-Ordered Tasks} & \multicolumn{4}{c}{\textcolor{black}{LetterWorld}}\tabularnewline
\hline 
\textcolor{black}{Number of Sub-goals} & \textcolor{black}{cDFAs} & \textcolor{black}{LTL2Action} & \textcolor{black}{RGTN} & \textcolor{black}{LOTUS}\tabularnewline
\hline 
\textcolor{black}{25} & \textcolor{black}{0.113 $\pm$ 0.316} & \textcolor{black}{0.806 $\pm$ 0.394} & \textcolor{black}{0.963 $\pm$ 0.187} & \textbf{\textcolor{black}{0.976 $\pm$ 0.103}}\tabularnewline
\textcolor{black}{26} & \textcolor{black}{0.080 $\pm$ 0.271} & \textcolor{black}{0.660 $\pm$ 0.473} & \textcolor{black}{0.930 $\pm$ 0.255} & \textbf{\textcolor{black}{0.938 $\pm$ 0.170}}\tabularnewline
\textcolor{black}{27} & \textcolor{black}{0.043 $\pm$ 0.203} & \textcolor{black}{0.556 $\pm$ 0.496} & \textcolor{black}{0.826 $\pm$ 0.378} & \textbf{\textcolor{black}{0.852 $\pm$ 0. 307}}\tabularnewline
\textcolor{black}{28} & \textcolor{black}{0.023 $\pm$ 0.150} & \textcolor{black}{0.416 $\pm$ 0.493} & \textcolor{black}{0.726 $\pm$ 0.445} & \textbf{\textcolor{black}{0.740 $\pm$ 0.438}}\tabularnewline
\textcolor{black}{29} & \textcolor{black}{0.010 $\pm$ 0.099} & \textcolor{black}{0.303 $\pm$ 0.459} & \textcolor{black}{0.683 $\pm$ 0.465} & \textbf{\textcolor{black}{0.716 $\pm$ 0.479}}\tabularnewline
\textcolor{black}{30} & \textcolor{black}{0.000 $\pm$ 0.000} & \textcolor{black}{0.230 $\pm$ 0.420} & \textcolor{black}{0.523 $\pm$ 0.499} & \textbf{\textcolor{black}{0.620 $\pm$ 0.490}}\tabularnewline
\textcolor{black}{31} & \textcolor{black}{0.000 $\pm$ 0.000} & \textcolor{black}{0.137 $\pm$ 0.344} & \textcolor{black}{0.376 $\pm$ 0.484} & \textbf{\textcolor{black}{0.434 $\pm$ 0.501}}\tabularnewline
\textcolor{black}{32} & \textcolor{black}{0.000 $\pm$ 0.000} & \textcolor{black}{0.067 $\pm$ 0.250} & \textcolor{black}{0.270 $\pm$ 0.443} & \textbf{\textcolor{black}{0.292 $\pm$ 0.393}}\tabularnewline
\textcolor{black}{33} & \textcolor{black}{0.000 $\pm$ 0.000} & \textcolor{black}{0.054 $\pm$ 0.226} & \textcolor{black}{0.165 $\pm$ 0.371} & \textbf{\textcolor{black}{0.196 $\pm$ 0.353}}\tabularnewline
\textcolor{black}{34} & \textcolor{black}{0.000 $\pm$ 0.000} & \textcolor{black}{0.011 $\pm$ 0.108} & \textcolor{black}{0.089 $\pm$ 0.285} & \textbf{\textcolor{black}{0.105 $\pm$ 0.207}}\tabularnewline
\textcolor{black}{35} & \textcolor{black}{0.000 $\pm$ 0.000} & \textcolor{black}{0.008 $\pm$ 0.094} & \textcolor{black}{0.044 $\pm$ 0.206} & \textbf{\textcolor{black}{0.063 $\pm$ 0.188}}\tabularnewline
\hline 
\end{tabular}}
\end{table}
\textcolor{black}{We conduct the following experiment to verify the
performance limits of RGTN on longer and more deeply nested LTL formulas.
As shown in Table \ref{tab:long_ltl_tasks_with_bs}, the observations
can be made: 1) With the same number of sub-goals, RGTN achieves higher
returns than cDFAs and LTL2Action, demonstrating the superiority of
RGTN's architecture in encoding LTL formulas that represent task semantics.
2) As the number of sub-objectives increases, all algorithms exhibit
varying degrees of performance degradation. When the number of sub-goals
exceeds 29, cDFAs struggles to complete tasks, while LTL2Action can
accomplish a few tasks but achieves less than half the return of RGTN.
These findings indicate that GAT used by cDFAs lacks sufficient semantic
representation capabilities for LTL tasks and has weak adaptability
to LTL formula structures. R-GCN exhibits limited representation capabilities,
lacking global semantic modeling capacity and demonstrating poor adaptability
in path-based generalization tasks. 3) When the number of sub-tasks
reaches 35, all methods struggle to complete the task, indicating
that it has reached the encoding limit of RGTN.
Future research integrating RGTN with knowledge graphs may
enhance its reasoning and encoding capabilities in more complex task
scenarios.}

\subsection{Real--World Experimental Results}

To further validate the effectiveness of LOTUS, we evaluated its performance on the Stack and Cleanup tasks from \cite{nasiriany2022augmenting} in real world settings. A Realsense D455 camera was used to capture object poses labeled with AprilTag \cite{olson2011apriltag} and a Franka Emika Panda robot executed the actions generated by the algorithm, as illustrated in Fig. \ref{fig:framework2lotus}(d).

The snapshots of two tasks are shown in Fig. \ref{fig:S2R_snapshot}, and a full video of the experiment is available on our \href{https://lotus-website.github.io/}{website}. 

\subsection{Limitations}

Experimental results demonstrate that LOTUS, supported by the RGTN architecture and direct updates, can effectively improve
agent performance across diverse task scenarios. However, several challenges remain. \textcolor{black}{First, due to
the expressive limitations of sc-LTL, LOTUS can only guide robots
in finite-horizon tasks and lacks $\omega$-regular completeness compared
to full LTL. Furthermore, this limitation prevents the LOTUS framework
from learning infinite-horizon tasks constructed using operators such
as $\square$(always), for example, patrol tasks. Research into integrating
B$\mathrm{\ddot{u}}$chi automaton and task-level representations
may extend the applicability of our method to infinite-horizon tasks.} \textcolor{black}{Second, decoupling the
LTL encoder from the RL gradient flow during the early stages of training
may lead to a misalignment between the downstream RL agent\textquoteright s
action performance and the task semantics extracted by the LTL encoder
due to unstable task representations. In addition to the bisimulation
metric method proposed in \ref{subsec:LBS2LTL_Encoder}, one potentially effective approach
is to jointly optimize the LTL encoder and the policy gradients
of the downstream RL agent.} Third, the proposed task representation algorithms must be trained from scratch to effectively guide agents across different scenarios. Incorporating LTL representations with large language models (LLMs) could further improve the training efficiency of the task representation module and enhance its guidance capabilities.

\section{CONCLUSIONS\label{sec:Conclusion}}

In this work, we introduce LOTUS, a temporal logic inspired universal task representation framework that can be seamlessly integrated into any existing RL algorithm to enhance agent performance across diverse task scenarios.
Within LOTUS, we introduce RGTN to encode LTL instructions, thereby improving agent performance and extracting task semantics. The LTL encoder is further treated as a policy, enabling reward-driven updates that enhance the efficiency of the task representation module. To improve stability, we leverage the bisimulation metric, which provides theoretical guarantees for LTL representations, including behavioral equivalence, optimality fidelity, and trajectory robustness. Extensive experiments demonstrate that LOTUS consistently outperforms existing methods.

\textcolor{black}{In our future work, we will focus on three key directions.
First, we will explore integrating existing work with automata-based
representations to enhance agents' performance in more complex task
scenarios (e.g., infinite horizon LTL task). Second, we will investigate
elevating the representational capacity of current LTL encoders to
LLM levels, enabling the agent to extract task semantics and thereby
improve learning efficiency across diverse task settings. 
% Finally,
% we will dedicate efforts to further advancing RGTN's graph encoding
% and reasoning capabilities. This foundation will enable the construction
% of corresponding knowledge graphs to improve agent' generalization
% performance.
\textcolor{black}{Finally, we will focus
on addressing the potential misalignment between the task representations
extracted by the LTL encoder and the performance of the downstream
agent. This improvement will further enhance the ability and efficiency
of LTL representations in guiding the agent to complete tasks.}
}

\bibliographystyle{IEEEtran}
\bibliography{Bib4RGTN}

\begin{IEEEbiography}
[{\includegraphics[width=1in,height=1.25in,clip,keepaspectratio]{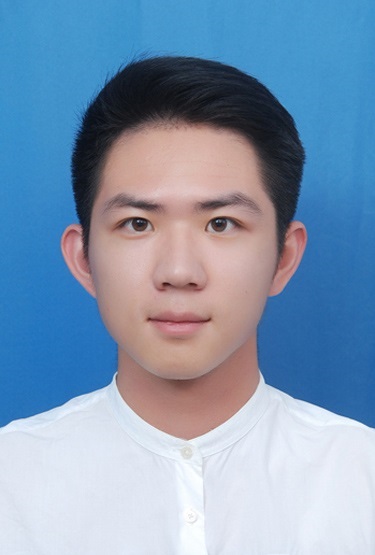}}]{Hao Zhang}
received the B.S. degree in mechanical engineering and automation from the Hefei
University of Technology, Hefei, Anhui, China, in
2020. He is currently pursuing the Ph.D. degree
in automation with the University of Science and
Technology of China, Hefei.

His current research interests include formal
methods in robotics, reinforcement learning, and
dexterous manipulation.
\end{IEEEbiography}

\begin{IEEEbiography}[{%
\includegraphics[width=1in,height=1.25in,keepaspectratio]{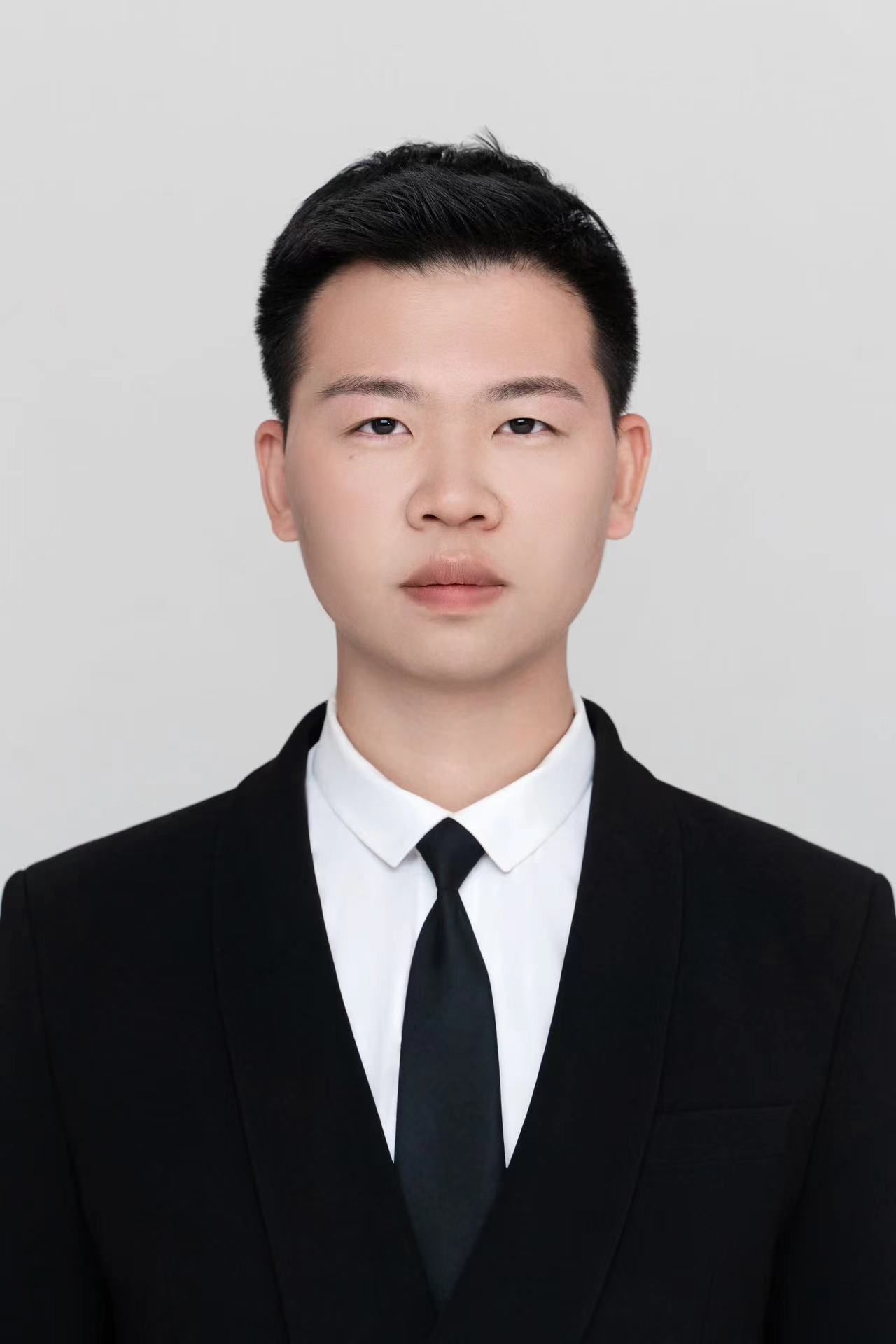}%
}]{Zhangli Zhou} received the B.E. degree from Wuhan University of Technology, in 2016, and the Ph.D. degree from the University of Science and Technology of China, in 2024. He is currently a Post-Doctoral Researcher with the University of Science and Technology of China, Hefei, China.
His research interests include robotics and reactive task and motion planning.
\end{IEEEbiography}
% \vspace{-3.0\baselineskip}

\begin{IEEEbiography}
[{\includegraphics[width=1in,height=1.25in,clip,keepaspectratio]{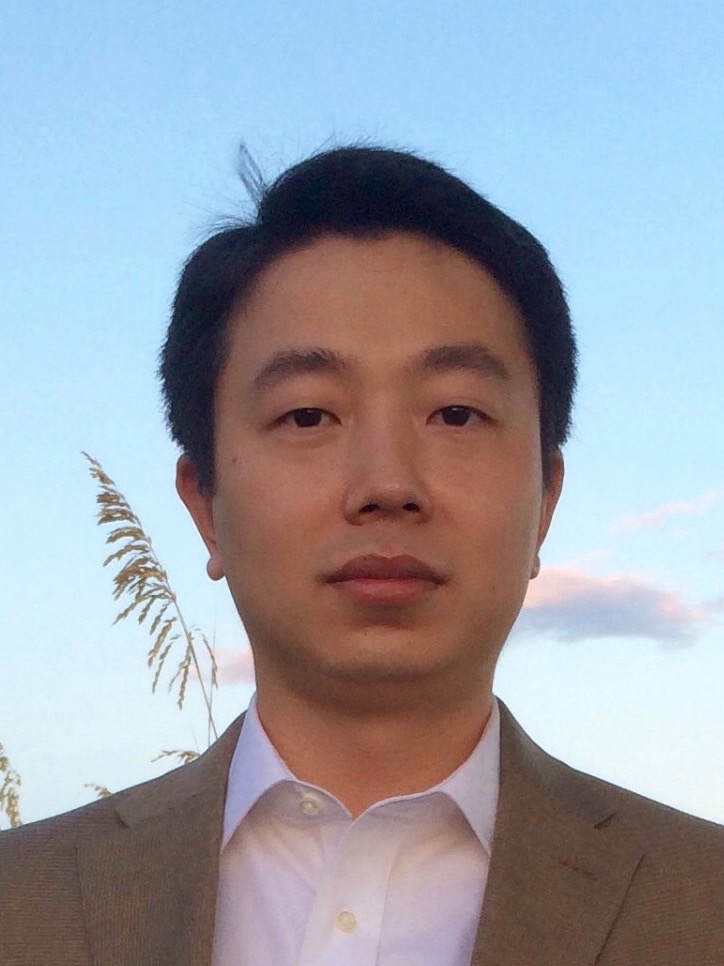}}]{Zhen Kan}
(Senior Member, IEEE) received the
Ph.D. degree from the Department of Mechanical
and Aerospace Engineering, University of Florida,
Gainesville, FL, USA, in 2011.

He is currently a Professor with the Department
of Automation, University of Science and
Technology of China, Hefei, China. His research interests include networked
control systems, nonlinear control, formal methods, and robotics. He currently serves on the program committees of several
internationally recognized scientific and engineering conferences and is an
Associate Editor of \textit{IEEE Transactions on Automatic Control}
\end{IEEEbiography}

\end{document}